\documentclass{article}
\usepackage{times}
\usepackage{graphicx} % more modern
\usepackage{subfigure} 
\usepackage[square,sort,comma,numbers]{natbib}
\usepackage{algorithm}
\usepackage{algorithmic}
\usepackage{booktabs}
\usepackage{hyperref}
\usepackage{nips15submit_e} 

\usepackage{color,amsmath,amssymb,multirow,wrapfig}

\newcommand{\inner}[1]{\left\langle#1\right\rangle}

\def\X{\mathcal{X}}

\def\R{\mathbb{R}}

\def\N{\mathbb{N}}

\def\Span{\textrm{Span}}
\newcommand{\norm}[1]{\left\|#1\right\|}

\def\sign{\mathop{\rm sign}\limits}
\def\maxop{\mathop{\rm max}\limits} %max operator
\def\minop{\mathop{\rm min}\limits}

\DeclareMathOperator{\arcsinh}{arcsinh}

\newenvironment{proof}{\par\noindent{\bf Proof:\ }}{\hfill$\Box$\\[2mm]}

\newtheorem{theorem}{Theorem}
\newtheorem{lemma}[theorem]{Lemma}

\def \R{{\mathbb R}}

\usepackage{xspace}
\makeatletter
\@ifundefined{onedot}{%
\DeclareRobustCommand\onedot{\futurelet\@let@token\@onedot}%
\def\@onedot{\ifx\@let@token.\else.\null\fi\xspace}%
}{}
\makeatother

\graphicspath{{.}{figures/}{images/}}
\DeclareGraphicsExtensions{.pdf,.png,.jpg}

\newcommand\ourmethodKL{{FMTL$_{kl}$} }

\newcommand\ourmethod{{FMTL$_p$} }
\newcommand\ourmethodHinge{{FMTL$_p$-H} }
\newcommand\ourmethodSquare{{FMTL$_p$-S} }

\newcommand\oursquared{{FMTL$_2$-S} }
    \makeatletter
    \let\@fnsymbol\@arabic
    \makeatother
    
\newif\iflongversion
\longversiontrue
\title{Efficient Output Kernel Learning for Multiple Tasks}

\author{
Pratik Jawanpuria\footnotemark[1]\ \ ,  Maksim Lapin\footnotemark[2]\ \ , Matthias Hein\footnotemark[1]\ \ \ and Bernt Schiele\footnotemark[2]\\
\footnotemark[1]\ \ Saarland University, Saarbr{\"u}cken, Germany\\
\footnotemark[2]\ \ Max Planck Institute for Informatics, Saarbr{\"u}cken, Germany%\\
}

\nipsfinalcopy
\begin{document}
\maketitle

\begin{abstract} 
% main points of the paper:
% - we study output kernel learning as a convex optimization problem
%   in order to leverage dependencies between the tasks
% - The output kernel has to be a positive definite matrix and thus the resulting optimization
%   problems are typically slow as they require an eigendecomposition at each step and thus
%   not scalable to a large number of tasks
% - we show that using the theory of positive semidefinite kernels that one can drop for a class
%   of regularizers the constraint of positive semidefiniteness in the optimization problem 
% - this leads to an easy form of the dual problem which can be solved very efficiently using
%   stochastic gradient descent basically at the cost of learning all tasks independently
% - experiments on several multi-task and multi-class datasets demonstrate that our approach
%   compares favorably with the state of the art

The paradigm of multi-task learning is that one can achieve better generalization by learning tasks jointly and thus exploiting the similarity between the tasks rather than learning them independently of each other. While previously the relationship between tasks had to be user-defined in the form of an output kernel, recent  approaches jointly learn the tasks and the output kernel. As the output kernel is a positive semidefinite matrix, the resulting optimization problems are not scalable in the  number of tasks as an eigendecomposition is required in each step. \mbox{Using} the theory of positive semidefinite kernels we show in this paper that for a certain class of regularizers on the output kernel, the constraint of being positive semidefinite can be dropped as it is automatically satisfied for the relaxed problem. This leads to an unconstrained dual problem which can be solved efficiently. 
Experiments on several multi-task and multi-class data sets illustrate the efficacy of our approach in terms of computational efficiency as well as generalization performance. 
% demonstrate that our approach is highly efficient and it obtains state-of-the-art generalization performance. 

\end{abstract} 

\section{Introduction}

Multi-task learning (MTL) advocates sharing relevant information among several related tasks during the training stage. The advantage of MTL over learning tasks independently has been shown theoretically as well as empirically~\citep{Evgeniou05,Argyriou08,Lounici09,Jalali10,mjaw11a,Maurer13,mjaw15}. %Various MTL methods have been proposed to learn what information should be shared among the tasks and/or how they should shared~\citep{Argyriou08,Jalali10,Evgeniou04,Lozano12,Lapin14}. 
%For instance multi-task feature learning approaches~\citep{Argyriou08,Jalali10,Lozano12,Lapin14} assume that the tasks share a common representation and aim to learn this feature subspace shared by the tasks. 
%This may be viewed as learning the features covariance, which is shared across all the tasks~\citep{Alvarez12}. 

The focus of this paper is the question how the task relationships can be inferred from the data.
It has been noted that naively grouping all the tasks together may be detrimental~\citep{Caruana97,Zhang10,Kang11,mjaw12}. 
In particular, outlier tasks may lead to worse performance. Hence, 
clustered multi-task learning algorithms~\citep{Kang11,Jacob08} aim to learn groups of closely related tasks. The information is then shared only within these clusters of tasks. This corresponds to learning the task covariance matrix, which we denote as the output kernel in this paper. Most of these approaches lead to non-convex problems.
%Approaches learning both both the feature and task covariance matrices have also been studied~\citep{Yi10,Archambeau11}. 

In this work, we focus on the problem of directly learning the output kernel in the multi-task learning framework. 
%The idea is that the output kernel captures the similarity between the tasks. 
The multi-task kernel on input and output is assumed to be decoupled as the product of a scalar kernel and the output kernel, which is a positive semidefinite matrix~\citep{Evgeniou05,Micchelli05,Caponnetto08,Alvarez12}.
% Having a good output kernel matrix is essential for the multi-task application since it controls the degree of relatedness among the %tasks. In the absence of any prior knowledge regarding task relatedness, it is essential to learn a good quality output kernel matrix from t%he data.
In classical multi-task learning algorithms~\citep{Evgeniou05,Evgeniou04}, the degree of relatedness between distinct tasks is set to a 
constant and is optimized as a hyperparameter. However, constant similarity between tasks is a strong assumption and is unlikely to hold in practice. Thus recent approaches have tackled the problem of directly learning the output kernel. 
\citep{Dinuzzo11} solves a multi-task formulation in the framework of vector-valued reproducing kernel Hilbert spaces involving squared loss where they penalize the Frobenius norm of the output kernel as a regularizer. They formulate an invex optimization problem that they solve optimally. 
In comparison, \citep{Ciliberto15} recently proposed an efficient barrier method to optimize a generic convex output kernel learning formulation. 
% This setting has been extended to learning a low rank output kernel (by imposing a rank constraint) in~\citep{Dinuzzo11a,Dinuzzo13}. 
% \citep{Zhang10} assumed the matrix-variate normal distribution over the task parameters and modeled task-covariance by the output kernel. The resulting formulation is non-convex and they solve its convex relaxation to learn a low rank output kernel matrix. 
On the other hand, \cite{Zhang10} proposes a convex formulation to learn low rank output kernel matrix by enforcing a trace constraint. 
% Interestingly, they show that their optimization problem is jointly convex both in the optimal weight vectors and the output kernel. 
% A trace constraint on the output kernel is imposed, enforcing low rank of the output kernel. 
% A related work is \citep{Dinuzzo11} where they solve a multi-task formulation in the framework of vector-valued reproducing kernel Hilbert spaces involving squared loss and where they penalize the Frobenius norm of the output kernel as a regularizer. They formulate an invex optimization problem that they can solve globally optimal. 
The above approaches~\citep{Zhang10,Dinuzzo11,Ciliberto15} solve the resulting optimization problem via alternate minimization between task parameters and the output kernel. Each step of the alternate minimization requires an eigenvalue decomposition of a matrix having as size the number of tasks and a problem corresponding to learning all tasks independently.

In this paper we study a similar formulation as \citep{Dinuzzo11}. However, we allow arbitrary convex loss functions and employ general $p$-norms for $p \in (1,2]$ (including the Frobenius norm) as regularizer for the output kernel. Our problem is jointly convex over the task parameters and the output kernel. Small $p$ leads to sparse output kernels which allows for an easier interpretation of the learned task relationships in the output kernel. 
Under certain conditions on $p$ we show that one can drop the constraint that the output kernel should be positive definite as it is automatically satisfied for the \mbox{unconstrained} problem. This significantly simplifies  the optimization and our result could also be of interest in other areas where one optimizes over the cone of positive definite matrices. 
The resulting unconstrained dual problem is amenable to efficient optimization methods such as stochastic dual coordinate ascent~\citep{Shalev-Shwartz13},
which scale well to large data sets.
%or sequential minimal optimization (SMO)~\citep{Platt99,Vishwanathan10}. 
Overall we do not require any eigenvalue decomposition operation at any stage of our algorithm and no alternate minimization is necessary, leading to a highly efficient methodology. Furthermore, we show that this trick not only applies to $p$-norms but also applies to a large class of regularizers for which we provide a characterization.
%
%Empirically, we compare our formulation with several existing multi-task techniques, including those whose focus is to learn the output kernel matrix, on benchmark multi-task and multi-class datasets. The results show that our formulation achieves state-of-the-art generalization performance and outperforms them in terms of training time. 

Our contributions are as follows: (a) we propose a generic $p$-norm regularized output kernel matrix learning formulation, which can be extended to a large class of regularizers; (b) we show that the constraint on the output kernel to be positive definite can be dropped as it is automatically satisfied,   leading to an unconstrained dual problem; (c) we propose an efficient stochastic dual coordinate ascent based method for solving the dual formulation; (d) we empirically demonstrate the superiority of our approach in terms of generalization performance as well as significant reduction in training time compared to other methods learning the output kernel. 

The paper is organized as follows. 
%In Section~\ref{sec:relatedWork}, we discuss the related works in the setting of output kernel learning. 
We introduce our formulation in Section~\ref{sec:oklf}. Our main technical result is discussed in Section~\ref{sec:specialDual}. The proposed optimization algorithm is described in Section~\ref{sec:optimization}. % In Section~\ref{sec:extensions}, we discuss a few extensions to our main result. 
In Section~\ref{sec:empiricalResults}, we report the empirical results. % and Section~\ref{sec:conclusion} concludes the paper.%discusses potential extensions. 
\iflongversion
\else
All the proofs can be found in the supplementary material.
\fi

%%%%%%%%%%%%%%%%%%%%%%%%%%%%%%%%%%%%%%%%%%%%%%%%%%%%%%%%%%%%%%%%%%%%%%%%%%%%%%%%%%%%%%%%%%%%%%%%%%%%%%%%%%%%%%%%%%%%%%%%%%%%%%%%%%%%%%%%%%%%%%%%%%%%%%%%%%%%%
%%%%%%%%%%%%%%%%%%%%%%%%%%%%%%%%%%%%%%%%%%%%%%%%%%%%%%%%%%%%%%%%%%%%%%%%%%%%%%%%%%%%%%%%%%%%%%%%%%%%%%%%%%%%%%%%%%%%%%%%%%%%%%%%%%%%%%%%%%%%%%%%%%%%%%%%%%%%%
%%%%%%%%%%%%%%%%%%%%%%%%%%%%%%%%%%%%%%%%%%%%%%%%%%%%%%%%%%%%%%%%%%%%%%%%%%%%%%%%%%%%%%%%%%%%%%%%%%%%%%%%%%%%%%%%%%%%%%%%%%%%%%%%%%%%%%%%%%%%%%%%%%%%%%%%%%%%%

\section{The Output Kernel Learning Formulation}\label{sec:oklf}
We first introduce the setting considered in this paper. We denote the number of tasks by $T$. We assume that all tasks have a common input space $\X$ and a common
positive definite kernel function $k:\X \times \X \rightarrow \R$. We denote by $\psi(\cdot)$ the feature map and by $H_k$ the reproducing kernel Hilbert space (RKHS)~\citep{Scholkopf02} associated with $k$. The training data is $(x_i,y_i,t_i)_{i=1}^n$, where $x_i \in \X$, $t_i$ is the task the $i$-th instance belongs to and $y_i$ is the corresponding label. Moreover, we have a positive definite matrix $\Theta \in S^T_+$ on the set of tasks $\{1,\ldots,T\}$,
 where $S^T_+$ is the set of $T\times T$ symmetric and positive semidefinite (p.s.d.) matrices.  

If one arranges the predictions of all tasks in a vector one can see multi-task learning as learning
a vector-valued function in a RKHS \citep[see][and references therein]{Evgeniou05,Micchelli05,Caponnetto08,Alvarez12,Ciliberto15}. However, in this paper we
use the one-to-one correspondence between real-valued and matrix-valued kernels, see \cite{HeiBou2004b}, in order to limit the technical overhead.
In this framework we define the joint kernel of 
input space and the set of tasks $M: (\X \times \{1,\ldots,T\}) \times (\X \times \{1,\ldots,T\}) \rightarrow \R$ as
\begin{equation}
M\big( (x,s), (z,t)\big)=k(x,z) \Theta(s,t),
% \[ M\big( (x,s), (z,t)\big)=k(x,z) \Theta(s,t),\] 
\end{equation}
We denote the corresponding RKHS of functions on $\X \times \{1,\ldots,T\}$ as $H_M$ and by $\norm{\cdot}_{H_M}$ the corresponding norm.
We formulate the output kernel learning problem for multiple tasks as
\begin{align}\label{eq:main}
 \minop_{\Theta\in S^T_+, F \in H_M} C\sum_{i=1}^n L\big(y_i,F(x_i,t_i)\big) + \frac{1}{2}\norm{F}^2_{H_M} + \lambda\,V(\Theta)
\end{align}
where  $L:\R \times \R \rightarrow \R$ is the convex loss function (convex in the second argument), $V(\Theta)$ is a convex regularizer
penalizing the complexity of the output kernel $\Theta$ and $\lambda \in \R_+$ is the regularization parameter. Note that $\norm{F}^2_{H_M}$
implicitly depends also on $\Theta$. In the following we show that \eqref{eq:main} can be reformulated into a jointly convex problem in the parameters of the prediction
function and the output kernel $\Theta$. 
\iflongversion
In order to see this  we first need the following representer theorem for fixed output kernel $\Theta$.
\begin{lemma}
The optimal solution $F^* \in H_M$ of the optimization problem
\begin{align}
 \minop_{F \in H_M} C\sum_{i=1}^n L\big(y_i,F(x_i,t_i)\big) + \frac{1}{2}\norm{F}^2_{H_M} 
\end{align}
admits a representation of the form 
\[ F^*(x,t) = \sum_{s=1}^T \sum_{i=1}^n \gamma_{is} M\big((x_i,s),(x,t)\big) = \sum_{s=1}^T \sum_{i=1}^n \gamma_{is} k(x_i,x) \Theta(s,t),\]
where $F^*(x,t)$ is the prediction for instance $x$ belonging to task $t$ and $\gamma \in \R^{n\times T}$.
\end{lemma}
\begin{proof}
The proof is analogous to the standard representer theorem \citep{Scholkopf02}. We denote by $U=\Span( M((x_i,s),(\cdot,\cdot))\,|\,i=1,\ldots,n, \, s=1,\ldots,T)$ the subspace in $H_M$
spanned by the training data. This induces the orthogonal decomposition of $H_M=U \oplus U^\perp$, where $U^\perp$ is the orthogonal subpace of $U$.
Every function $F \in H_M$ can correspondingly decomposed into $F=F^{\parallel}+F^\perp$, where $F^{\parallel} \in U$ and $F^{\perp} \in U^\perp$.
Then $\norm{F}_{H_M}^2 = \norm{F^{\parallel}}^2_{H_M}+\norm{F^\perp}^2_{H_M}$. As 
\begin{equation}
F(x_i,t_i) = \inner{F,M((x_i,t_i),(\cdot,\cdot))} = \inner{F^\parallel,M((x_i,t_i),(\cdot,\cdot))}=F^\parallel(x_i,t_i). 
\end{equation}
As the loss only depends on $F^\parallel$ and we minimize the objective by having $\norm{F^\perp}_{H_M}=0$. This yields the result.
\end{proof}
\else
Using the standard representer theorem \cite{Scholkopf02} (see the supplementary material) for fixed output kernel $\Theta$, one can show that the optimal solution $F^* \in H_M$
of %the optimization problem 
\eqref{eq:main} can be written as
\begin{equation}
F^*(x,t) = \sum_{s=1}^T \sum_{i=1}^n \gamma_{is} M\big((x_i,s),(x,t)\big) = \sum_{s=1}^T \sum_{i=1}^n \gamma_{is} k(x_i,x) \Theta(s,t). 
\end{equation}
\fi
With the explicit form of the prediction function one can rewrite the main problem \eqref{eq:main} as 
\begin{align}\label{eq:main2}
 \minop_{\Theta\in S^T_+, \gamma \in \R^{n \times T}} C\sum_{i=1}^n L\big(y_i,\sum_{s=1}^T\sum_{j=1}^n \gamma_{js} k_{ji}\Theta_{s\,t_i}\big) 
 + \frac{1}{2}\sum_{r,s=1}^T \sum_{i,j=1}^n \gamma_{ir} \gamma_{js} k_{ij}\Theta_{rs} + \lambda\,V(\Theta),
\end{align}
where $\Theta_{rs}=\Theta(r,s)$ and $k_{ij}=k(x_i,x_j)$.
Unfortunately, problem \eqref{eq:main2} is not jointly convex in $\Theta$ and $\gamma$ due to the product in the second term. A similar problem has 
been analyzed in \citep{Dinuzzo11}. They could show that for the squared loss and $V(\Theta)=\norm{\Theta}_F^2$ the corresponding optimization problem
is invex and directly optimize it. For an invex function every stationary point is globally optimal \cite{IsrMon1986}. %They then propose a direct method to optimize the invex problem.

We follow a different path which leads to a formulation similar to the one of \cite{Argyriou08} used for learning an input mapping (see also \cite{Zhang10}). Our formulation for the output kernel learning problem is jointly convex in the task kernel $\Theta$ and the task parameters. 
We present a derivation for the general RKHS $H_k$, analogous to the linear case presented in~\cite{Argyriou08,Zhang10}. 
% The derivation is for infinite dimensional RKHS $H_k$ and thus we present the derivation to keep the paper self-contained. 
We use the following variable transformation,
\[ \beta_{it} = \sum_{s=1}^T \Theta_{ts}\gamma_{is}, \; i=1,\ldots,n,\; s=1,\ldots,T,\quad \textrm{ resp. } \quad \gamma_{is} = \sum_{t=1}^T \big(\Theta^{-1}\big)_{st} \beta_{it}.\]
In the last expression $\Theta^{-1}$ has to be understood as the pseudo-inverse if $\Theta$ is not invertible. Note that this causes no problems as in case  $\Theta$ is not invertible,
we can without loss of generality restrict $\gamma$ in \eqref{eq:main2} to the range of $\Theta$. The transformation leads to our final problem formulation, where the prediction function $F$ and its squared norm $\norm{F}^2_{H_M}$ can be written as
\begin{equation}\label{eq:pred-beta}
 F(x,t) %=  \sum_{s=1}^T \sum_{i=1}^n \gamma_{is} k(x_i,x) \Theta(s,t) 
 = \sum_{i=1}^n \beta_{it} k(x_i,x), \qquad \norm{F}^2_{H_M} = \sum_{r,s=1}^T \sum_{i,j=1}^n  \big(\Theta^{-1}\big)_{sr} \beta_{is}\beta_{jr} k(x_i,x_j).
\end{equation}
%Plugging the transformed variable in the expression for $\norm{F}^2_{H_M}$ we get
%\[ \norm{F}^2_{H_M} = \sum_{r,s=1}^T \sum_{i,j=1}^n  \big(\Theta^{-1}\big)_{sr} \beta_{is}\beta_{jr} k(x_i,x_j).\]
\iflongversion
This can be seen as follows
\begin{align}
\norm{F}^2_{H_M} &= \sum_{r,s=1}^T \sum_{i,j=1}^n \gamma_{ir} \gamma_{js} k(x_i,x_j)\Theta_{rs}\\    
                 &= \sum_{t,u=1}^T \sum_{r,s=1}^T \sum_{i,j=1}^n  \beta_{it}\beta_{ju} \big(\Theta^{-1}\big)_{tr}\big(\Theta^{-1}\big)_{us} k(x_i,x_j)\Theta_{rs}\\
                 &= \sum_{t,u=1}^T \sum_{i,j=1}^n \big(\Theta^{-1}\big)_{tu} \beta_{it}\beta_{ju} k(x_i,x_j).
\end{align} 
\fi               
We get our final primal optimization problem
\begin{align}\label{eq:main3}
 %\minop_{\Theta\in S^T_+, \beta \in \R^{n\times T}} C\sum_{i=1}^n L\big(y_i,\sum_{j=1}^n \beta_{j t_i} k(x_j,x_i)\big) + \frac{1}{2}\sum_{t,t'=1}^T \sum_{r,s=1}^T \sum_{i,j=1}^n %\big(\Theta^{-1}\big)_{ts} \beta_{is}\big(\Theta^{-1}\big)_{t'r} \beta_{jr} k(x_i,x_j)\Theta_{tt'} + \lambda\,V(\Theta)\\
 \minop_{\Theta\in S^T_+, \beta \in \R^{n\times T}} C\sum_{i=1}^n L\big(y_i,\sum_{j=1}^n \beta_{j t_i} k_{ji}\big) + \frac{1}{2} \sum_{r,s=1}^T \sum_{i,j=1}^n 
\big(\Theta^{-1}\big)_{sr} \beta_{is}\beta_{jr} k_{ij}+ \lambda\,V(\Theta)
\end{align}
Before we analyze the convexity of this problem, we want to illustrate the connection to the formulations in \cite{Zhang10,Dinuzzo11}. With the task weight vectors $w_t=\sum_{j=1}^n \beta_{jt} \psi(x_j) \in H_k$ we get predictions as $F(x,t)=\inner{w_t,\psi(x)}$ and one can rewrite 
\[ \norm{F}^2_{H_M} = \sum_{r,s=1}^T \sum_{i,j=1}^n  \big(\Theta^{-1}\big)_{sr} \beta_{is}\beta_{jr} k(x_i,x_j) = \sum_{r,s=1}^T  \big(\Theta^{-1}\big)_{sr} \inner{w_s,w_t}.\]
This identity is known for vector-valued RKHS, see \cite{Alvarez12} and references therein.
% It is instructive to discuss an important special case. 
When $\Theta$ is $\kappa$ times the identity matrix, then $\norm{F}^2_{H_M}=\sum_{t=1}^T 
% \norm{w_t}^2/\kappa$ and thus \eqref{eq:main} is learning the tasks independently. 
\frac{\norm{w_t}^2}{\kappa}$ and thus \eqref{eq:main} is learning the tasks independently. 
% \footnote{In the given formulation, the tasks will be learned independently up to the common regularization parameter $C$, which can also be set differently for various tasks.}
As mentioned before the convexity of the expression of $\norm{F}^2_{H_M}$ is crucial for the convexity of the full problem  \eqref{eq:main3}. The following result has been shown in \cite{Argyriou08} (see also \cite{Zhang10}).
\begin{lemma} 
Let $\mathcal{R}(\Theta)$ denote the range of $\Theta \in S^T_+$ and let $\Theta^\dagger$ be the pseudoinverse.
The extended function $f:S^T_+ \times \R^{n\times T} \rightarrow \R \cup \{\infty\}$ defined as
\[ f(\Theta,\beta)=\begin{cases} \sum_{r,s=1}^T \sum_{i,j=1}^n \big(\Theta^{\dagger}\big)_{sr} \beta_{is}\beta_{jr} k(x_i,x_j), & \textrm{ if }\beta_{i\cdot} \in \mathcal{R}(\Theta), \forall \, i=1,\ldots,n,\\ \infty & \textrm{ else }. \end{cases},\]
 is jointly convex.
\end{lemma}
\iflongversion
\begin{proof}
It has been shown in \cite{Argyriou08} and \cite{Dat2008}[p. 223] that $\inner{x,A^\dagger x}$ is jointly convex on $S^T_+ \times \mathcal{R}(A)$, where $\mathcal{R}(A)$ is the
range of $A$ and $A^\dagger$ is the pseudoinverse of $A \in S^T_+$. As $L:=(k(x_i,x_j))_{i,j=1}^n$ is positive semi-definite we can compute the eigendecomposition as
\[ L_{ij} = \sum_{l=1}^n \lambda_l u_{li} u_{lj},\]
where $\lambda_l\geq 0$, $l=1,\ldots,n$ are the eigenvalues and $u_l \in \R^n$ the eigenvectors. Using this we get 
\begin{align}
\sum_{r,s=1}^T \sum_{i,j=1}^n 
\big(\Theta^{-1}\big)_{sr} \beta_{is}\beta_{jr} k(x_i,x_j) = \sum_{l=1}^n \lambda_l \sum_{r,s=1}^T  \Big(\sum_{i=1}^n \beta_{is}u_{li}\Big) \Big(\sum_{j=1}^n \beta_{jr}u_{lj}\Big) \big(\Theta^{-1}\big)_{rs}
\end{align}
and thus we can write the function $f$ as a positive combination of convex functions, where the arguments are composed with linear mappings which preserves convexity
\cite{Boyd04}.
\end{proof}
\fi
The formulation in \eqref{eq:main3} is similar to ~\citep{Zhang10,Dinuzzo11,Ciliberto15}. \cite{Zhang10} uses the constraint $\mathrm{Trace}(\Theta)\leq 1$ instead
of a regularizer $V(\Theta)$ enforcing low rank of the output kernel. % (see also \cite{Dinuzzo11a,Dinuzzo13}). 
On the other hand, \citep{Dinuzzo11} employs squared Frobenius norm for $V(\Theta)$ with squared loss function. 
\citep{Ciliberto15} proposed an efficient algorithm for convex $V(\Theta)$. 
Instead we think that sparsity of $\Theta$
is better to avoid the emergence of spurious relations between tasks and also leads to output kernels which are easier to interpret. Thus we propose to use 
the following regularization functional for the output kernel $\Theta$: 
\begin{equation*}%\label{eqn:regularizer}
V(\Theta)=\sum_{t,t'=1}^T |\Theta_{tt'}|^p =\norm{\Theta}_p^p,
\end{equation*}
for $p \in [1,2]$. Several approaches \citep{Zhang10,Dinuzzo11,Ciliberto15} employ alternate minimization scheme, involving costly eigendecompositions of $T\times T$ matrix per iteration (as $\Theta\in S_+^T$). In the next section we show that for a certain set of values of $p$ one can derive an unconstrained
dual optimization problem which thus avoids the explicit minimization over the $S_+^T$ cone. 
% The trick to get this result is to use results from linear algebra characterizing the of functions which when applied elementwise to a psd matrix yield again a psd matrix. 
The resulting \emph{unconstrained} dual problem can then be easily optimized by
stochastic coordinate ascent. Having explicit expressions of the primal variables $\Theta$ and $\beta$ in terms of the dual variables allows us to get back
to the original problem.

%The effect of this regularization on the output kernel $\Theta$ is as follows: as $p$ tend towards one, $\Theta$ matrix will become sparser and thus have fewer dominant terms. This can help to avoid learning spurious correlations between the tasks. Note that the Frobenius norm of \citep{Dinuzzo11} is a special case of our regularization ($p=2$). We will show in the next section that with the above choice of $V(\Theta)$, the number of hyper-parameter can be reduced to one, thereby reducing cross-validation time. 
%%We will show in the next section that in certain cases, the regularization parameter $\lambda$ in the optimization problem \label{eq:main} is redundant and thus only the parameter $C$ has to be cross-validated. 
%
%%In this case, it can also be shown that regularization parameter $\lambda$ is redundant. %{\red The special case of $p=1$ is discussed %in section~\ref{sec:extensions}.}
%%{\red Should $p$-norm be given more discussion?}

\section{Unconstrained Dual Problem Avoiding Optimization over $S^T_+$}\label{sec:specialDual}
% We have derived in the previous section a jointly convex primal formulation \eqref{eq:main3} of the multi-task output kernel learning problem. 
The primal formulation \eqref{eq:main3} is a convex multi-task output kernel learning problem. 
The next lemma derives the Fenchel dual function of \eqref{eq:main3}. This still involves the optimization over the primal variable $\Theta\in S_+^T$.
A main contribution of this paper is to show that this optimization problem over the $S_+^T$ cone can be solved with an analytical solution
for a certain class of regularizers $V(\Theta)$. In the following we denote by $\alpha^r :=\{\alpha_i \,|\,t_i=r\}$ the dual variables corresponding to task $r$
and by $K_{rs}$ the kernel matrix $(k(x_i,x_j) \,|\, t_i=r, t_j=s)$ corresponding to the dual variables of tasks $r$ and $s$.
\begin{lemma}\label{le:Fenchel}
Let $L_{i}^*$ be the conjugate function of the loss $L_{i}: \R \rightarrow \R, u \mapsto L(y_{i},u)$, then   % with $\br(\Theta)=\norm{\Theta}_{p}$. 
%  In particular, their objective values are equal at optimality. 
\begin{align}\label{eqn:outputKernelDual1}
 %\minop_{\Theta\in S^T_+} 
q:\R^{n} \rightarrow \R, q(\alpha) =  
 -C\sum_{i=1}^n L_{i}^*\Big(-\frac{\alpha_{i}}{C}\Big) - \lambda \maxop_{\Theta\in S^T_+}  \Big(
   \frac{1}{2\lambda}\sum_{r,s=1}^T \Theta_{rs}\inner{\alpha^r, K_{rs}\alpha^{s}}  - V(\Theta)\Big)
 %\maxop_{\alpha\in \R^{mT}}\   -C\sum_{t=1}^T \sum_{i=1}^m L_{ti}^*\big(-\frac{\alpha_{ti}}{C}\big) - \frac{1}{2}\sum_{t,t'=1}^T \Theta_{tt'}\alpha_t^\top K_{tt'}\alpha_{t'},
%   + \frac{\lambda}{2}\Big(\sum_{t,t'=1}^T |\Theta_{tt'}|^p\Big)^\frac{2}{p},
 \end{align}
is the dual function of \eqref{eq:main3}, where $\alpha \in \R^{n}$ are the dual variables.
The primal variable $\beta \in \R^{n \times T}$ in \eqref{eq:main3} and the prediction function $F$ can be expressed in terms of $\Theta$ and $\alpha$ as
%\[ w_s = \sum_{r=1}^T \sum_{i=1}^m \Theta_{sr} \alpha_{ri}\phi(\bx_{ri}), \quad s=1,\ldots,T.\]
% \[ \beta_{is} = \alpha_i \Theta_{s t_i}, \qquad  F(x,s) = \sum_{j=1}^n \alpha_j \Theta_{s t_j} k(x_j,x) ,\]
$\beta_{is} = \alpha_i \Theta_{s t_i}$ and $F(x,s) = \sum_{j=1}^n \alpha_j \Theta_{s t_j} k(x_j,x)$ respectively, 
where $t_j$ is the task of the $j$-th training example.
\end{lemma}
\iflongversion
\begin{proof} 
We derive the Fenchel dual function of \eqref{eq:main3}. For this purpose we introduce auxiliary variables $z \in \R^{n}$ which satisfy the constraint
\[ z_{i}=\sum_{j=1}^n \beta_{j t_i} k(x_j,x_i)=F(x_i,t_i).\]
The Lagrangian $L$ of the resulting problem \eqref{eq:main3} is given as:
\begin{align}
 L(\beta,\Theta,z,\alpha) =& \ C\sum_{i=1}^n L(y_{i},z_{i}) + \frac{1}{2} \sum_{r,s=1}^T \sum_{i,j=1}^n 
\big(\Theta^{-1}\big)_{sr} \beta_{is}\beta_{jr} k(x_i,x_j)  \\
  & + \sum_{i=1}^n \alpha_{i}\Big(z_i - \sum_{j=1}^n \beta_{j t_i} k(x_j,x_i)\Big)
    + i_{S^T_+}(\Theta) + \lambda\,V(\Theta).\nonumber
\end{align}
where $i_C$ is the indicator function of the set $C$.
The dual function $q$ is defined as
\begin{equation}\label{eq:dual-function}
 q(\alpha) = \minop_{\beta \in \R^{n \times T}, \, \Theta \in S^T_+,\,z \in \R^n} L(\beta,\Theta,z,\alpha).
\end{equation}
Using the definition of the conjugate function ~\citep{Boyd04}, we get
\begin{align}\label{eqn:lossFenchel}
 \minop_{z_i \in \R} C\,L(y_i,z_i) +\alpha_{i}z_{i} &= C\minop_{z_i \in \R} L(y_i,z_i) +\frac{\alpha_i}{C}z_{i} 
                                                     = -C\,\maxop_{z_{i} \in \R} \Big(-\frac{\alpha_{i}}{C}z_{i}- L(y_i,z_i)\Big) \\
                                                     &=  -C\,L_{i}^*\big(-\frac{\alpha_{i}}{C}\big),
\end{align}
where $L^*_{i}$ is the conjugate function of $L_{i}:z\rightarrow L(y_{i},z)$. Moreover, we compute the minimizer with respect to $\beta$, via
\begin{align}
&\frac{\partial }{\partial \beta_{lu}} \Big(\frac{1}{2} \sum_{r,s=1}^T \sum_{i,j=1}^n 
\big(\Theta^{-1}\big)_{sr} \beta_{is}\beta_{jr} k(x_i,x_j)  -\sum_{i=1}^n \alpha_{i}\big(\sum_{j=1}^n \beta_{j t_i} k(x_j,x_i)\Big)\\
=& \sum_{r=1}^T \sum_{j=1}^n \beta_{jr} (\Theta^{-1})_{ur} k(x_l,x_j) - \sum_{i=1}^n \alpha_i \delta_{u t_i} k(x_l,x_i),\nonumber
\end{align}
where $\delta$ is the Kronecker symbol, that is $\delta_{u t_i} = \begin{cases} 1 & \textrm{ if } u=t_i,\\ 0 & \textrm{ else } \end{cases}$.
Solving for the global minimizer $\beta^*$ yields 
% ({\red TODO}: here invertibility of $\Theta$ and the kernel matrix is required)
\begin{equation}\label{eq:beta-alpha}
 \beta^*_{jr} = \alpha_j \Theta_{r t_j}.
\end{equation}
Plugging $\beta^*$ back into the above expressions yields
\begin{align}
\sum_{r,s=1}^T \sum_{i,j=1}^n 
\big(\Theta^{-1}\big)_{sr} \beta_{is}\beta_{jr} k(x_i,x_j) 
&= \sum_{r,s=1}^T \sum_{i,j=1}^n 
\big(\Theta^{-1}\big)_{sr}\Theta_{s t_i} \Theta_{r t_j} \alpha_i \alpha_j  k(x_i,x_j)\nonumber\\
&=\sum_{i,j=1}^n \Theta_{t_i t_j} \alpha_i \alpha_j  k(x_i,x_j),\\
 \sum_{i,j=1}^n \alpha_{i}\beta_{j t_i} k(x_j,x_i)  &= \sum_{i,j=1}^n \alpha_i \alpha_j \Theta_{t_j t_i} k(x_j,x_i),
\end{align}
Introducing $\alpha^r=(\alpha_{i})_{t_i=r}$, $K_{rs}=\big(k(x_i,x_j)\big)_{t_i=r, t_j=s}$ and gathering the terms corresponding to 
the individual tasks we get
\[ \sum_{i,j=1}^n \alpha_i \alpha_j \Theta_{t_j t_i} k(x_j,x_i) = \sum_{r,s=1}^T \inner{\alpha^r, K_{rs} \alpha^{s}}.\]
%Note that $\bw_t^\top\phi(\bx_{ti})=\phi(\bx_{ti})^\top\bW\be_t$.
Plugging all the expressions back into \eqref{eq:dual-function}, we get the dual function as
\begin{align}
 q(\alpha) &= -C\,L_{ti}^*(-\frac{\alpha_{ti}}{C}) + \minop_{\Theta \in S^T_+} \lambda\,V(\Theta) - \frac{1}{2}\sum_{r,s=1}^T \Theta_{rs}\inner{\alpha^r, K_{rs} \alpha^s}\\
           &= -C\,L_{ti}^*(-\frac{\alpha_{ti}}{C}) + \lambda \minop_{\Theta \in S^T_+}  V(\Theta) - \inner{\rho,\Theta}\\
           &= -C\,L_{ti}^*(-\frac{\alpha_{ti}}{C}) - \lambda \maxop_{\Theta \in S^T_+}  \inner{\rho,\Theta} - V(\Theta) 
\end{align}
where we have introduced in the second step $\rho \in \R^{T \times T}$ with
\[ \rho_{rs} = \frac{1}{2\lambda} \inner{\alpha^r, K_{rs} \alpha^s}, \quad r,s=1,\ldots,T.\]
Note that $\rho$ is a Gram matrix and thus positive semidefinite. The expression for the prediction function is obtained by plugging
\eqref{eq:beta-alpha} into \eqref{eq:pred-beta}.
%The prediction function of a task $t$ is given by $F_{t'}(x) = \sum_{t,i}\alpha_{ti}K(x,x_{ti})\Theta_{t't}$.
\end{proof}
\fi
%\else
%Note that $K_{tt'}$ need not be p.s.d. matrix for $t\neq t'$. The derivation of the dual along with the expressions for the primal variables are detailed in the supplementary material. 
%\fi
% The derivation of the dual follows from the representer theorem~\citep{Scholkopf02}. This along with the expressions for the primal variables are detailed in the supplementary material. 

We now focus on the remaining maximization problem in the dual function in \eqref{eqn:outputKernelDual1}
\begin{equation} \label{eq:conjugate}
  \maxop_{\Theta\in S^T_+}  \frac{1}{2\lambda}\sum_{r,s=1}^T \Theta_{rs}\inner{\alpha^r, K_{rs}\alpha^{s}}  -  V(\Theta).
\end{equation}
This is a semidefinite program which is computationally expensive to solve and thus prohibits to scale the output kernel learning problem
to a large number of tasks. However, we show in the following that this problem has an analytical solution for a subset of the regularizers $V(\Theta)=\frac{1}{2}\sum_{r,s=1}^T |\Theta_{rs}|^p$ for $p\geq 1$. For better readability we defer a more general result towards the end of the section.
The basic idea is to relax the constraint on $\Theta \in R^{T \times T}$ in \eqref{eq:conjugate} so that it is equivalent to the computation of the conjugate $V^*$ of $V$.
If the maximizer of the relaxed problem is positive semi-definite, one has found the solution of the original problem.
\begin{theorem}\label{th:duality}
Let $k \in \N$ and $p=\frac{2k}{2k-1}$, then with $\rho_{rs}=\frac{1}{2\lambda}\inner{\alpha^r, K_{rs}\alpha^{s}}$ we have 
\begin{align}\label{eq:dual-pnorm}
 \maxop_{\Theta\in S^T_+} \sum_{r,s=1}^T \Theta_{rs}\rho_{rs}  - \frac{1}{2} \sum_{r,s=1}^T |\Theta_{rs}|^p 
\,=\,\frac{1}{4k-2} \Big(\frac{2k-1}{2k\lambda}\Big)^{2k} \sum_{r,s=1}^T \inner{\alpha^r, K_{rs}\alpha^{s}}^{2k},
\end{align}
and the maximizer is given by the positive semi-definite matrix
\begin{align}\label{eq:dual-pnorm-maximizer}
 \Theta^*_{rs} = \Big(\frac{2k-1}{2k\lambda}\Big)^{2k-1} \inner{\alpha^r, K_{rs}\alpha^{s}}^{2k-1}, \quad r,s=1,\ldots,T.
\end{align}
\end{theorem}
\iflongversion
\begin{proof}
We relax the constraints and solve
\[ \maxop_{\Theta\in \R^{T \times T}}  \frac{1}{2\lambda}\sum_{r,s=1}^T \Theta_{rs}\inner{\alpha^r, K_{rs}\alpha^{s}}  - \frac{1}{2} \sum_{r,s=1}^T |\Theta_{rs}|^p. \]
Note that the problem is separable and thus we can solve for each component separately,
\[ \maxop_{\Theta_{rs} \in \R} \frac{1}{2\lambda}\Theta_{rs}\inner{\alpha^r, K_{rs}\alpha^{s}}  -  \frac{1}{2} |\Theta_{rs}|^p.\]
%We define the matrix $\rho_{rs}=\inner{\alpha^r, K_{rs}\alpha^{s}}$
The optimality condition for $\Theta^*_{rs}$ becomes with $\rho_{rs}=\frac{1}{2\lambda}\inner{\alpha^r, K_{rs}\alpha^{s}}$,
\[ 0 = \rho_{rs} - \frac{p}{2} \sign(\Theta^*_{rs})|\Theta^*_{rs}|^{p-1} \;\, \Longrightarrow \;\, \Theta^*_{rs} = \Big(\frac{2}{p}\Big)^\frac{1}{p-1} \,\sign(\rho_{rs})|\rho_{rs}|^\frac{1}{p-1}.\]
The solution of the relaxed problem is the solution of the original constrained problem, if we can show that the corresponding maximizer is positive semidefinite.
Note that $\rho_{rs}=\frac{1}{2\lambda}\inner{\alpha^r, K_{rs}\alpha^{s}}$ is a positive semidefinite (p.s.d.) matrix as it is a Gram matrix. The factor $\Big(\frac{2}{p}\Big)^\frac{1}{p-1}$ 
is positive and thus the resulting matrix is p.s.d. if $\sign(\rho_{rs})|\rho_{rs}|^\frac{1}{p-1}$ is p.s.d.

It has been shown \cite{Horn1969}, that the elementwise power $A_{rs}^l$ of a positive semidefinite matrix $A$ is positive definite for all $A \in S^T_+$ and $T \in \N$ 
if and only if $l$ is a positive integer. Note that we have an elementwise integer power of $\Theta$ if $\frac{1}{p-1}$ is an odd positive integer (the case of an even integer is ruled out by Theorem \ref{th:schoenberg}), that is $\frac{1}{p-1}=2k-1$ for $k \in \N$ as in this
case we have
\[ \Theta^*_{rs} = \Big(\frac{2}{p}\Big)^{2k-1} \,\sign(\rho_{rs})|\rho_{rs}|^{2k-1} = \Big(\frac{2}{p}\Big)^{2k-1} \rho_{rs}^{2k-1} =  \Big(\frac{2k-1}{2k\lambda}\Big)^{2k-1} \inner{\alpha^r,K_{rs}\alpha^s}^{2k-1}  .\]
We get the admissible values of $p$ as $p=\frac{2k}{2k-1}$, $k \in \N$ (resp. $2k=\frac{p}{p-1}$). We compute the optimal objective value as 
\begin{align}
\sum_{r,s=1}^T \rho_{rs}^{2k} \Big( \Big(\frac{2}{p}\Big)^{2k-1} - \frac{1}{2} \Big(\frac{2}{p}\Big)^{2k}\Big)&=(p-1)\frac{1}{2} \Big(\frac{2}{p}\Big)^{2k} \sum_{r,s=1}^T \rho_{rs}^{2k} = \frac{1}{4k-2}\Big(\frac{2k-1}{k}\Big)^{2k}\sum_{r,s=1}^T \rho_{rs}^{2k}\\
&=  \frac{1}{4k-2}\Big(\frac{2k-1}{2\lambda\,k}\Big)^{2k} \sum_{r,s=1}^T\inner{\alpha^r,K_{rs}\alpha^s}^{2k}
\end{align}
\end{proof}
\fi
Plugging the result of the previous theorem into the dual function of Lemma \ref{le:Fenchel} we get for $k \in \N$ and 
$p=\frac{2k}{2k-1}$ with $V(\Theta)=\norm{\Theta}_p^p$ the following unconstrained dual of our main problem \eqref{eq:main3}:
\begin{equation}\label{eq:final-dual}
\maxop_{\alpha \in \R^n} -C\sum_{i=1}^n L_{i}^*\Big(-\frac{\alpha_{i}}{C}\Big) - \frac{\lambda}{4k-2} \Big(\frac{2k-1}{2k\lambda}\Big)^{2k} \sum_{r,s=1}^T \inner{\alpha^r,K_{rs}\alpha^s}^{2k}.
\end{equation}
Note that by doing the variable transformation $\kappa_i:=\frac{\alpha_i}{C}$ we effectively have only one hyper-parameter in \eqref{eq:final-dual}.  This allows us to cross-validate more efficiently.
The range of admissible values for $p$ in Theorem \ref{th:duality} lies in the interval $(1,2]$, where we get for $k=1$ the value $p=2$ and as $k\rightarrow \infty$ we have $p\rightarrow 1$. The regularizer for $p=2$ together with the squared loss has been considered in the primal in \cite{Dinuzzo11,Ciliberto15}. Our analytical expression of the dual is novel and allows us to employ stochastic dual coordinate ascent to solve the involved primal optimization problem. Please also note that by optimizing the dual, we have access to the duality gap and thus a well-defined
stopping criterion. This is in contrast to the alternating scheme of \cite{Dinuzzo11,Ciliberto15} for the primal problem which involves costly matrix operations. 
Our runtime experiments show that our solver for \eqref{eq:final-dual} outperforms the solvers of \cite{Dinuzzo11,Ciliberto15}. Finally, note that even for suboptimal dual
variables $\alpha$, the corresponding $\Theta$ matrix in \eqref{eq:dual-pnorm-maximizer} is positive semidefinite. Thus we always get a feasible set of
primal variables.

% \begin{center}
\begin{table}\centering
{%\small
\caption{Examples of regularizers $V(\Theta)$ together with their generating function $\phi$ and the explicit form of $\Theta^*$ in terms of the dual
variables, $\rho_{rs}=\frac{1}{2\lambda}\inner{\alpha^r, K_{rs}\alpha^{s}}$. The optimal value of \eqref{eq:conjugate} is given in terms of $\phi$ as $\maxop_{\Theta \in \R^{T \times T}} \inner{\rho,\Theta} - V(\Theta) = \sum_{r,s=1}^T \phi(\rho_{rs})$.
}\label{tab:examples}
\begin{tabular}{l|l|l}
\hline
$\phi(z)$         & $V(\Theta)$ %& $\maxop_{\Theta \in \R^{T \times T}} \inner{\rho,\Theta} - V(\Theta)$ 
& $\Theta_{rs}^*$\\
\hline
$\frac{z^{2k}}{2k}$,\, $k\in \N$ & $\frac{2k-1}{2k}\sum\limits_{r,s=1}^T |\Theta_{rs}|^\frac{2k}{2k-1}$ %& $\frac{1}{2k}\sum\limits_{r,s=1}^T |\Theta|^{2k}$ 
& $\rho_{rs}^{2k-1}$\\
\hline
$e^z=\sum_{k=0}^\infty \frac{z^k}{k!}$                         & $\begin{cases} \sum\limits_{r,s=1}^T \Theta_{rs} \log(\Theta_{rs}) - \Theta_{rs} & \textrm{ if } \Theta_{rs}>0 \forall r,s\\ \infty & \textrm{ else .}\end{cases}$ %& $\sum\limits_{r,s=1}^T e^{\rho_{rs}}$ 
& $e^{\rho_{rs}}$\\
\hline
$\cosh(z)-1=\sum_{k=1}^\infty \frac{z^{2k}}{(2k)!}$ &  
$\sum\limits_{r,s=1}^T \Big(\Theta_{rs} \arcsinh(\Theta_{rs}) - \sqrt{1+\Theta_{rs}^2}\Big) +T^2$ %& $\sum\limits_{r,s=1}^T (\cosh(\rho_{rs})-1)$
& $\arcsinh(\rho_{rs})$\\
\hline
\end{tabular}
}
\end{table}
% \end{center}

\paragraph{Characterizing the set of convex regularizers $V$ which allow an analytic 
expression for the dual function}
The previous theorem raises the question for which class of convex, separable regularizers we can get an analytical expression of the dual function by explicitly solving the optimization problem \eqref{eq:conjugate} over the positive semidefinite cone. A key element in the proof of the previous theorem is the characterization of functions $f:\R \rightarrow \R$ which when applied elementwise $f(A)=(f(a_{ij}))_{i,j=1}^T$ to a positive semidefinite matrix $A \in S^T_+$ result in a p.s.d. matrix, that is $f(A) \in S^T_+$. This set of functions has been characterized by Hiai \cite{Hiai09}.
\begin{theorem}[\cite{Hiai09}]\label{th:schoenberg} 
Let $f:\R \rightarrow \R$ and $A \in S^T_+$. We denote by $f(A)=(f(a_{ij}))_{i,j=1}^T$ the elementwise application of $f$ to $A$.
It holds 
% \[ \forall \,T \geq 2, \quad A \in S^T_+ \quad \Longrightarrow \quad f(A) \in S^T_+,\]
$\forall \,T \geq 2, \quad A \in S^T_+ \Longrightarrow f(A) \in S^T_+$
if and only if $f$ is analytic and $f(x)=\sum_{k=0}^\infty a_k x^k$ with $a_k\geq 0$ for all $k\geq 0$.
\end{theorem}
Note that in the previous theorem the condition on $f$ is only necessary when we require the implication to hold for all $T$. If $T$ is fixed, the set of functions is larger 
and includes even (large) fractional powers, see \cite{Horn1969}. We use the stronger formulation as we want that the result holds without any restriction on the number of tasks $T$. Theorem \ref{th:schoenberg} is the key element used in our following characterization of separable regularizers of $\Theta$
which allow an analytical expression of the dual function.
\vspace*{-10pt}
\begin{theorem}\label{th:schoenberg1}
Let $\phi:\R \rightarrow \R$ be analytic on $\R$ and given as $\phi(z)=\sum_{k=0}^\infty \frac{a_k}{k+1}z^{k+1}$ where $a_k\geq 0\ \forall k\geq 0$.
If $\phi$ is convex, then, $V(\Theta):=\sum_{r,s=1}^T \phi^*(\Theta_{rs})$,
is a convex function $V:\R^{T \times T} \rightarrow \R$ and 
\begin{align}\label{eq:conj2}
\maxop_{\Theta \in \R^{T \times T}} \inner{\rho,\Theta} - V(\Theta) 
\,=\, V^*(\rho) \,=\, \sum_{r,s=1}^T \phi\big(\rho_{rs}\big),
\end{align}
where the global maximizer fulfills $\Theta^* \in S^T_+$ if $\rho \in S^T_+$ and $\Theta^*_{rs} = \sum_{k=0}^\infty a_k \rho_{rs}^k.$
\end{theorem}
\iflongversion
\begin{proof}
Note that $\phi$ is analytic on $\R$ and thus infinitely differentiable on $\R$. As $\phi$ is additionally convex, it is a proper, lower semi-continuous convex function
and thus $(\phi^*)^*=\phi$ \cite[Corollary 1.3.6]{HirLem2001}. As $\phi^*$ is convex, $V$ is a convex function and using $(\phi^*)^*=\phi$ we get
\begin{align}\label{eq:phiconj}
 \maxop_{\Theta \in \R^{T \times T}} \inner{\rho,\Theta} - V(\Theta) =  V^*(\rho)= \sum_{r,s=1}^T \phi(\rho_{rs}).
\end{align}
Finally, we show that the global maximizer has the given form. Note that as $\phi$ is a proper, lower semi-continuous convex function it holds \cite[Corollary 1.4.4]{HirLem2001}
\[ \Theta_{rs} \in \partial \phi^*(\rho_{rs}) \quad \Longleftrightarrow \quad  \rho_{rs} \in \partial\phi(\Theta_{rs}).\]
Note that the maximizer $\Theta_{rs}^*$ of problem \eqref{eq:phiconj} fulfills $\rho_{rs} \in \frac{\partial \phi^*}{\partial \Theta_{rs}}(\Theta^*_{rs})$ and thus $\Theta_{rs}^*=\frac{\partial \phi}{\partial \rho_{rs}}(\rho_{rs})$,
where we have used that $\phi$ is infinitely differentiable.
These conditions allow us to express the maximizer of \eqref{eq:conj2} in terms of $\partial \phi$. 
As $\phi$ is continuously differentiable, we get 
\[ \Theta_{rs}^*=\frac{\partial \phi}{\partial \rho_{rs}}(\rho_{rs}) = \sum_{k=0}^\infty a_k \rho_{rs}^k.\]
Note that the series has infinite convergence radius and $a_k\geq 0$ for all $k$ and thus it is of the form provided in Theorem \ref{th:schoenberg}. Thus $\Theta^* \in S^T_+$
if $\rho \in S^T_+$.
\end{proof}
\fi
Table \ref{tab:examples} summarizes e.g. of functions $\phi$, the corresponding $V(\Theta)$ and the maximizer $\Theta^*$ in  \eqref{eq:conj2}.
% We summarize in Table \ref{tab:examples} examples of functions $\phi$, the corresponding regularizers and the form of the maximizer $\Theta^*$ of  \eqref{eq:conj2}.
% We summarize in Table \ref{tab:examples} examples of functions $\phi$ and the corresponding regularizers, value of \eqref{eq:conj2}, form of the maximizer $\Theta^*$.

\iflongversion
\paragraph{Examples}
\begin{itemize}
\item First we recover the results of Theorem \ref{th:duality}. We use $\phi(x)=\frac{1}{2k} x^{2k}$ for $k \in \N$, which is convex.
      We compute 
      \[ \phi^*(y) =\sup_{x \in \R} xy - \phi(x) = \sup_{x \in \R} xy - \frac{1}{2k} x^{2k} = \frac{2k-1}{2k} |y|^\frac{2k}{2k-1},\]
      where we have used $x^*=|y|^\frac{1}{2k-1}\mathrm{sign}(y)$. We recover
      \[ V(\Theta) = \sum_{r,s=1}^T \phi^*(\Theta_{rs}) = \frac{2k-1}{2k} \sum_{r,s=1}^T \Theta^\frac{2k}{2k-1}_{rs},\]
      which with $p=\frac{2k}{2k-1}$ yields up to a positive factor the family of regularizers employed in Theorem \ref{th:duality} together with
      \[ \Theta_{rs}^* = \rho_{rs}^{2k-1} \]
%\item Next we study what we get when using odd powers. However, odd powers are not convex, so we have to truncate at zero.
%      We use $\phi(x)=\begin{cases} 0 & \textrm{ if } x<0\\, \frac{1}{2k+1} x^{2k+1} & x\geq 0 \end{cases}$ for $k \in \N$.
%      DOES NOT WORK AS THIS FUNCTION IS NOT ANALYTIC !!!
%      This function is convex and  we compute 
%      \[ \phi^*(y) =\sup_{x \in \R} xy - \phi(x) = \begin{cases} \frac{2k}{2k+1} y^\frac{2k+1}{2k}  & \textrm{ if } y\geq 0,\\ \infty & \textrm{ else }.\end{cases}.\]
%      We get
%      \[ V(\Theta) = \sum_{r,s=1}^T \phi^*(\Theta_{rs}) = \begin{cases} \sum_{r,s=1}^T \Theta^\frac{2k+1}{2k}_rs & \textrm{ if } \Theta_{rs}\geq 0,\; \forall r,s=1,\ldots,T,\\
%      \infty & \textrm{ else }\end{cases}.\]
%      We get $\Theta_{rs}^* = \gamma_{rs}^{2k}$ which preserves positive semi-definiteness      
\item In the second example we use $\phi(x)=e^x = \sum_{k=0}^\infty \frac{x^k}{k!}$, which is convex and the series has infinite convergence radius
      The conjugate $\phi^*$ is given as
      \[ \phi^*(y) =\sup_{x \in \R} xy - e^x = \begin{cases} y \log(y)-y & \textrm{ if } y>0\\ \infty & \textrm{ else}. \end{cases}\]
      so that the regularizer is given by,
      \[ V(\Theta) = \sum_{r,s=1}^T \phi^*(\Theta_{rs}) = \begin{cases} \sum_{r,s=1}^T \Theta_{rs} \log(\Theta_{rs}) - \Theta_{rs} & \textrm{ if } \Theta_{rs}>0\ \forall r,s=1,\ldots,T\\ \infty & \textrm{ else .}\end{cases}.\]
      This can be seen as a generalized KL-divergence between $\Theta$ and $\Theta_0$, where $\Theta_0 \in S^T_+$ is the matrix of all ones
      \[ V(\Theta) = \sum_{r,s=1}^T \phi^*(\Theta_{rs}) = \begin{cases} \sum_{r,s=1}^T \Theta_{rs} \log\left(\frac{\Theta_{rs}}{\big(\Theta_0\big)_{rs}}\right) - \Theta_{rs} + \big(\Theta_0\big)_{rs} & \textrm{ if } \Theta_{rs}>0\ \forall r,s\\ \infty & \textrm{ else .}\end{cases}.\]
      Note that adding the constant term $\sum_{r,s=1}^T \big(\Theta_0\big)_{rs}$ does not change the optimization problem \eqref{eq:main3}.
      The corresponding $\Theta^*$ is given by
      \[ \Theta_{rs}^* = \sum_{k=0}^\infty \frac{\rho_{rs}^k}{k!} = e^{\rho_{rs}}.\]
\item Next we consider $\phi(x)=\cosh(x)-1=\sum_{k=1}^\infty \frac{x^{2k}}{(2k)!}$ which is obviously convex and the series has infinite convergence radius ($e^x$ is majorant).
      The conjugate $\phi^*$ can be computed as
      \[ \phi^*(y) = \sup_{x \in \R} xy - \cosh(x)+1 = y \arcsinh(y) - \sqrt{1+y^2} +1 = y \log(y+\sqrt{y^2+1}) - \sqrt{1+y^2}+1.\]
      so that the regularizer is given by 
      \[  V(\Theta) = \sum_{r,s=1}^T \phi^*(\Theta_{rs}) = \sum_{r,s=1}^T \Big(\Theta_{rs} \arcsinh(\Theta_{rs}) - \sqrt{1+\Theta_{rs}^2} +1\Big).\]
      The corresponding $\Theta^*$ is given by
      \[ \Theta^*_{rs} = \arcsinh(\rho_{rs})=\log\big(\rho_{rs}+\sqrt{\rho_{rs}^2+1}\big).\]
      This regularizer is interpolating between a squared norm and a variant of $1$-norm. One has
      \[ \lim_{y \rightarrow 0} \phi^*(y) = \frac{y^2}{2}, \quad \lim_{y \rightarrow \infty} \phi^*(y)=|y|(\log(2|y|)-1)+1.\]
%\item In the last case we use $\phi_{rs}(x)=\frac{1}{2}x^2 + \phi^0_{rs}x$, where $\phi^0 \in S^T_+$. We have
%      \[ \phi^*(y) = \sup_{x \in \R} xy - \frac{1}{4}x^2 - \phi^0_{rs}x = (y-\phi^0_{rs})^2.\]
%      The regularizer becomes
%      \[ V(\Theta) = \sum_{r,s=1}^T \phi^*_{rs}(\Theta_{rs}) = \norm{\Theta-\Theta^0}_F^2.\]
%      The optimal $\Theta^*$ is given by
%      \[ \Theta^*_{rs} = \rho_{rs} + \phi^0_{rs}.\]
\end{itemize}
\fi

\begin{algorithm}[tb]
   \caption{\footnotesize Fast MTL-SDCA}
   \label{alg:sdca}
\begin{algorithmic}\footnotesize
   \STATE {\bfseries Input:} Gram matrix $K$, label vector $y$, regularization parameter and relative duality gap parameter $\epsilon$
   \STATE {\bfseries Output:} $\alpha$ ($\Theta$ is computed from $\alpha$ using our result in ~\ref{eq:dual-pnorm-maximizer})
   \STATE Initialize $\alpha=\alpha^{(0)}$
    \REPEAT 
      \iflongversion
      \STATE Let $\{i_1,\ldots,i_n\}$ be a random permutation of $\{1,\ldots,n\}$
      \FOR{$j=1,\ldots,n$}
      \fi
      \iflongversion
      \STATE Solve for $\Delta$ in (\ref{eqn:smallProblem}) corresponding to $\alpha_{i_j}$
      \STATE $\alpha_{i_j}\leftarrow\alpha_{i_j} + \Delta$
      \else
      \STATE Randomly choose a dual variable $\alpha_{i}$
      \STATE Solve for $\Delta$ in (\ref{eqn:smallProblem}) corresponding to $\alpha_{i}$
      \STATE $\alpha_{i}\leftarrow\alpha_{i} + \Delta$
%       \STATE $d=d+1$
      \fi
      \iflongversion
      \ENDFOR
%       \STATE $d=d+1$
      \fi
    \UNTIL{Relative duality gap is below $\epsilon$}
\end{algorithmic}
\end{algorithm}
\section{Optimization Algorithm}\label{sec:optimization}
The dual problem~\eqref{eq:final-dual} can be efficiently solved via decomposition based methods like stochastic dual coordinate ascent algorithm (SDCA)~\citep{Shalev-Shwartz13}. SDCA enjoys low computational complexity per iteration and has been shown to scale effortlessly to large scale optimization problems. 

\iflongversion
Our algorithm for learning the output kernel matrix and task parameters is summarized in Algorithm~\ref{alg:sdca}. 
\else
Our algorithm for learning the output kernel matrix and task parameters is summarized in Algorithm~\ref{alg:sdca} (refer to the supplementary material for more details). 
\fi
At each step of the iteration we optimize the dual objective over a randomly chosen $\alpha_i$ variable. 
% Note that as we solve the dual problem, we have a well defined stopping criterion via the duality gap since we can easily compute from our current dual solution a primal feasible point. 
% 
% In SDCA, one randomly choses a dual variable (coordinate) and optimizes~\eqref{eq:final-dual}  with respect to it while fixing all other variables. We accordingly update this dual variable and repeat the process till convergence is achieved. 
% Note that as we solve the dual problem, we have a well defined stopping criterion via the duality gap since we can easily compute from our current dual solution a primal feasible point. 
Let $t_i=r$ be the task \mbox{corresponding} to $\alpha_i$. We apply the update $\alpha_{i}\leftarrow\alpha_{i}+\Delta$. The optimization problem of solving~\eqref{eq:final-dual} with respect to $\Delta$ is as follows:
\begin{align}\label{eqn:smallProblem}
% \minop_{\Delta\in\R} &  \quad\Big(\left(\substack{a\Delta^2+\\2b_{itt}\Delta+c_{tt}}\right)^q + 2\sum_{t'\neq t}\Big(\substack{b_{itt'}\Delta+\\c_{tt'}}\right)^q+\sum_{t',t''\neq t}c_{t't''}^q\Big)^\frac{2}{q}\nonumber\\
%  & + 8\lambda C L_{ti}^*\left(-\frac{\alpha_{ti}+\Delta}{C}\right),
  \scriptsize\minop_{\Delta\in\R}\ L_{i}^*\big((-\alpha_{i}-\Delta)/C\big) +
    \eta{\big((a\Delta^2+ 2b_{rr}\Delta+c_{rr})^{2k} + 2\sum_{s\neq r} (b_{rs}\Delta+ c_{rs})^{2k}+\sum_{s,z\neq r} c_{sz}^{2k}\big),}
\end{align}
where $a=k_{ii}$, $b_{rs}= \sum_{j:t_j=s} k_{ij}\alpha_{j}\ \forall s$, $c_{sz}=\inner{\alpha^s,K_{sz}\alpha^z}\ \forall s,z$ and $\eta=\frac{\lambda}{C(4k-2)} \Big(\frac{2k-1}{2k\lambda}\Big)^{2k}$. This one-dimensional convex optimization problem is solved efficiently via Newton method. 
%which took only 2-3 iterations in our experiments. 
The complexity of the proposed algorithm is $O(T)$ per iteration .  
% We cache the terms $b_{rs}$ and $c_{sz}$ for additional computational efficiency. 
%($(m+1)T^2$ entries) and update relevant values at every iteration. 
The proposed algorithm can also be employed for learning output kernels regularized by generic $V(\Theta)$, discussed in the previous section.
%We also report the results obtained with $V(\Theta)$ being the KL-divergence based regularizer (row 2 in Table~\ref{tab:examples}) in the next section. \\
% Since it requires the computation of exponential function, its runtime is slower than the $p$-norm based $V(\Theta)$. 

\textbf{Special case $p=2(k=1)$}: For certain loss functions such as the hinge loss, the squared loss, etc., $L_{ti}^*\big(-\frac{\alpha_{ti}+\Delta}{C}\big)$ yields a linear or a quadratic expression in $\Delta$. In such cases problem~\eqref{eqn:smallProblem} reduces to finding the roots of a cubic equation, which has a closed form expression. Hence, our algorithm is highly efficient with the above loss functions when $\Theta$ is regularized by the squared Frobenius norm.

\section{Empirical Results}\label{sec:empiricalResults}
In this section, we present our results on benchmark data sets comparing our algorithm with existing approaches in terms of generalization accuracy as well as computational efficiency. 
\iflongversion
In Section~\ref{subsec:multi-task-exp}, we discuss generalization results in multi-task setting. We evaluate the performance of our algorithm against several recent multi-task methods that employ clustering, low-dimensional projection of input feature space or output kernel learning. Section~\ref{subsec:multi-class-exp} discusses multi-class experiment results. Single task learning (\textbf{STL}) is a common baseline in both these experiments, and it employs hinge loss and $\epsilon$-SVR loss functions for classification and regression problems respectively. Finally, in Section~\ref{sec:scaling}, we discuss the results on the computational efficiency of our algorithm. 
\else
Please refer to the supplementary material for additional results and details. 
\fi

\subsection{Multi-Task Data Sets}\label{subsec:multi-task-exp}
We begin with the generalization results in multi-task setups. 
\iflongversion
The data sets are as follows:\\
\textbf{Sarcos:} A multi-task regression data set. The aim is to predict 7 degrees of freedom of a robotic arm~\citep{Zhang09b}. \\
\textbf{Parkinson:} A multi-task regression data set~\citep{Frank10} where one needs to predict the Parkinson's disease symptom score for 42 patients.\\
\textbf{Yale:} A face recognition data set from the Yale face base with 28 binary classification tasks~\citep{Yi10}.\\
\textbf{Landmine:} A data set containing binary classification problems from 19 different landmines~\citep{Yi10}. \\
\textbf{MHC-I:} A bioinformatics data set having 10 binary classification tasks~\citep{Jacob08}.\\
\textbf{Letter:} A data set containing handwritten letters from several writers and having 9 binary classification tasks~\citep{Ji09}.

\else
The data sets are as follows: 
a) \textbf{Sarcos}: a regression data set, aim is to predict 7 degrees of freedom of a robotic arm, 
b) \textbf{Parkinson:} a regression data set, aim is to predict the Parkinson's disease symptom score for 42 patients,
c) \textbf{Yale:} a face recognition data with 28 binary classification tasks, 
d) \textbf{Landmine:} a data set containing binary classifications from 19 different landmines, 
e) \textbf{MHC-I:} a bioinformatics data set having 10 binary classification tasks, 
f) \textbf{Letter:} a handwritten letters data set with 9 binary classification tasks.\\
\fi
\iflongversion
\begin{table}\centering
\caption{Dataset statistics.
$T$ represents the number of tasks and
$m$ represents the average number of training examples per task.
}\label{table:dataStatistics}
\vspace*{1em}
\begin{tabular}{lll||lll}\toprule
 Dataset & $T$ & $m$ &  Dataset & $T$ & $m$ \\
\midrule
\midrule
 {Sarcos} & $7$ & $15$ &  {Landmine} & $19$ & $102$\\
 {Parkinson} & $42$ & $5$ &  {MHC-I} & $10$ & $24$\\
 {Yale} & $28$ & $5$ &  {Letter} & $9$ & $60$\\
 {USPS} & $10$ & $100$ & {MNIST} & $10$ & $100$\\
 {MIT Indoor67} & $67$ & $80$ & {SUN397} & $397$ & $5$, $50$\\
\bottomrule
\end{tabular}
\end{table}

Table~\ref{table:dataStatistics} presents the data set statistics.
% \FloatBarrier
\else
% The data set statistics is provided in the supplementary material. 
\fi
\iflongversion
We compare the following algorithms:\\
% \textbf{STL}: Single task learning baseline.\\
\textbf{MTL}~\cite{Evgeniou04}: A classical multi-task learning baseline. They define the $\Theta$ matrix as: $\Theta(t,t')=\frac{1}{\mu}+\delta_{tt'}$, where $\mu>0$ is a hyper-parameter and $\delta_{tt'}=1$ if $t=t'$ else $\delta_{tt'}=0$. The hyper-parameter $\mu$ is cross-validated.\\
\textbf{CMTL}~\citep{Jacob08}: A clustered multi-task learning algorithm. Tasks within a cluster are assumed to be close to a mean vector. It requires the number of task clusters as a hyper-parameter. \\
% \textbf{DMTL}~\citep{Jalali10}: A multi-task feature learning baseline. It does not learn the output kernel.\\
\textbf{MTFL}~\citep{mjaw12}: Learns the input kernel and the output kernel matrix as a linear combination of base kernel matrices. \\%It also incorporates input kernel learning.\\
\textbf{GMTL}~\citep{Kang11}: A clustered multi-task feature learning approach. Tasks within a cluster are assumed to share a low dimensional feature  subspace~\citep{Argyriou08}. Hence, it effectively learns both the input kernel as well as the output kernel. \\
\textbf{MTRL}~\cite{Zhang10}: A multi-task relationship learning approach. It learns a low rank output kernel matrix by enforcing a trace constraint on it.\\
\textbf{FMTL$_p$}: Our proposed multi-task learning formulation~\eqref{eq:final-dual}. We consider three different values for the $p$-norm: $p=2\ (k=1)$, $p=4/3\ (k=2)$ and $p=8/7\ (k=4)$. Hinge and $\epsilon$-SVR loss functions were used for classification and regression problems respectively.

\fi
\iflongversion
We follow the experimental protocol\footnote{The performance of STL, MTL, CMTL and MTFL are reported from~\cite{mjaw12}. 
%The train-test splits were provided by the authors.
} described in~\cite{mjaw12}. Three-fold cross validation was performed for parameter selection. Linear kernel was employed for all data sets. Also, note that GMTL~\cite{Kang11} and MTFL~\cite{mjaw12} enjoy the advantage of both input and output kernel learning. Hence, their generalization results are not directly comparable to our method, which focuses solely on learning the output kernel matrix. 

\fi

Table~\ref{table:resultPerf} reports the performance of the algorithms averaged over ten random train-test splits. 
The proposed FMTL$_p$ attains the best generalization accuracy in general. 
It outperforms the baseline MTL as well as MTRL and CMTL, which solely learns the output kernel matrix.
Moreover, it achieves an overall better performance than GMTL and MTFL. %, which learn the input kernel in addition to the output kernel matrix. 
The FMTL$_{p=4/3,8/7}$ give comparable generalization to $p=2$ case, with the additional benefit of learning sparser and more interpretable output kernel matrix  (see Figure~\ref{fig:theta-landmine}). 
% We can note the influence of the $q$-norm on {SARCOS} data set, where $q=4$ and $q=8$ performed better than the Frobenius norm $(q=2)$ regularization. 
%On other data sets as well, the generalization achieved with $q=4$ and $q=8$ is comparable to that obtained with $q=2$.%, with the advantage of learning a sparse output kernel matrix. 
\begin{table*}\centering\scriptsize
\caption{%Performance of various methods on regression and classification datasets.
Mean generalization performance and the standard deviation over ten train-test splits.
%Our proposed \ourmethod is overall the best method. % in 4 out of 6 data sets.
}\label{table:resultPerf}
% \vspace*{0.75em}
\renewcommand{\tabcolsep}{0.45em}
\begin{tabular}{l|c|ccccc|ccc}\toprule
\multirow{2}{*}{\textbf{Data set}} &
\multirow{2}{*}{\textbf{STL}} &
\multirow{2}{*}{\textbf{MTL}} &
\multirow{2}{*}{\textbf{CMTL}} &
% \multirow{2}{*}{\textbf{DMTL}} &
\multirow{2}{*}{\textbf{MTFL}} &
\multirow{2}{*}{\textbf{GMTL}} &
\multirow{2}{*}{\textbf{MTRL}} &
\multicolumn{3}{c}{\textbf{\ourmethod}}\\
 & & & & & & & $p=2$ & $p=4/3$ & $p=8/7$ \\
\midrule
\midrule
\multicolumn{10}{l}{Regression data sets: Explained Variance (\%)}\vspace*{.1em} \\
{Sarcos} &
$40.5{\pm}7.6$ & $34.5{\pm}10.2$ & $33.0{\pm}13.4$ & $49.9{\pm}6.3$ &
$45.8{\pm}10.6$ & $41.6{\pm}7.1$ & $46.7{\pm}6.9$ & $\mathbf{50.3{\pm}5.8}$ & $48.4{\pm}5.8$ \\
{Parkinson} &
$2.8{\pm}7.5$ & $4.9{\pm}20.0$ & $2.7{\pm}3.6$ & $16.8{\pm}10.8$ &
$\mathbf{33.6{\pm}9.4}$ & $12.0{\pm}6.8$ & $27.0{\pm}4.4$ & $27.0{\pm}4.4$ & $27.0{\pm}4.4$ \\
\midrule
\multicolumn{10}{l}{Classification data sets: AUC (\%)}\vspace*{.1em} \\
{Yale} &
$93.4{\pm}2.3$ & $96.4{\pm}1.6$ & $95.2{\pm}2.1$  & $\mathbf{97.0{\pm}1.6}$ &
$91.9{\pm}3.2$ & $96.1{\pm}2.1$ & $\mathbf{97.0{\pm}1.2}$ & $\mathbf{97.0{\pm}1.4}$ & $96.8{\pm}1.4$ \\
{Landmine} &
$74.6{\pm}1.6$ & $76.4{\pm}0.8$ & $75.9{\pm}0.7$ & $76.4{\pm}1.0$ &
$76.7{\pm}1.2$ & $76.1{\pm}1.0$ & $\mathbf{76.8{\pm}0.8}$ & $76.7{\pm}1.0$ & $76.4{\pm}0.9$ \\
{MHC-I} &
$69.3{\pm}2.1$ & $72.3{\pm}1.9$ & $\mathbf{72.6{\pm}1.4}$ & $71.7{\pm}2.2$ &
$72.5{\pm}2.7$ & $71.5{\pm}1.7$ & $71.7{\pm}1.9$ & $70.8{\pm}2.1$ & $70.7{\pm}1.9$ \\
{Letter} &
$61.2{\pm}0.8$ & $61.0{\pm}1.6$ & $60.5{\pm}1.1$ & $60.5{\pm}1.8$ &
$61.2{\pm}0.9$ & $60.3{\pm}1.4$ & $61.4{\pm}0.7$ & $\mathbf{61.5{\pm}1.0}$ & $61.4{\pm}1.0$ \\
\bottomrule
\end{tabular}
\end{table*}

\begin{figure}[t]\small\centering
\includegraphics[width=0.3\columnwidth]{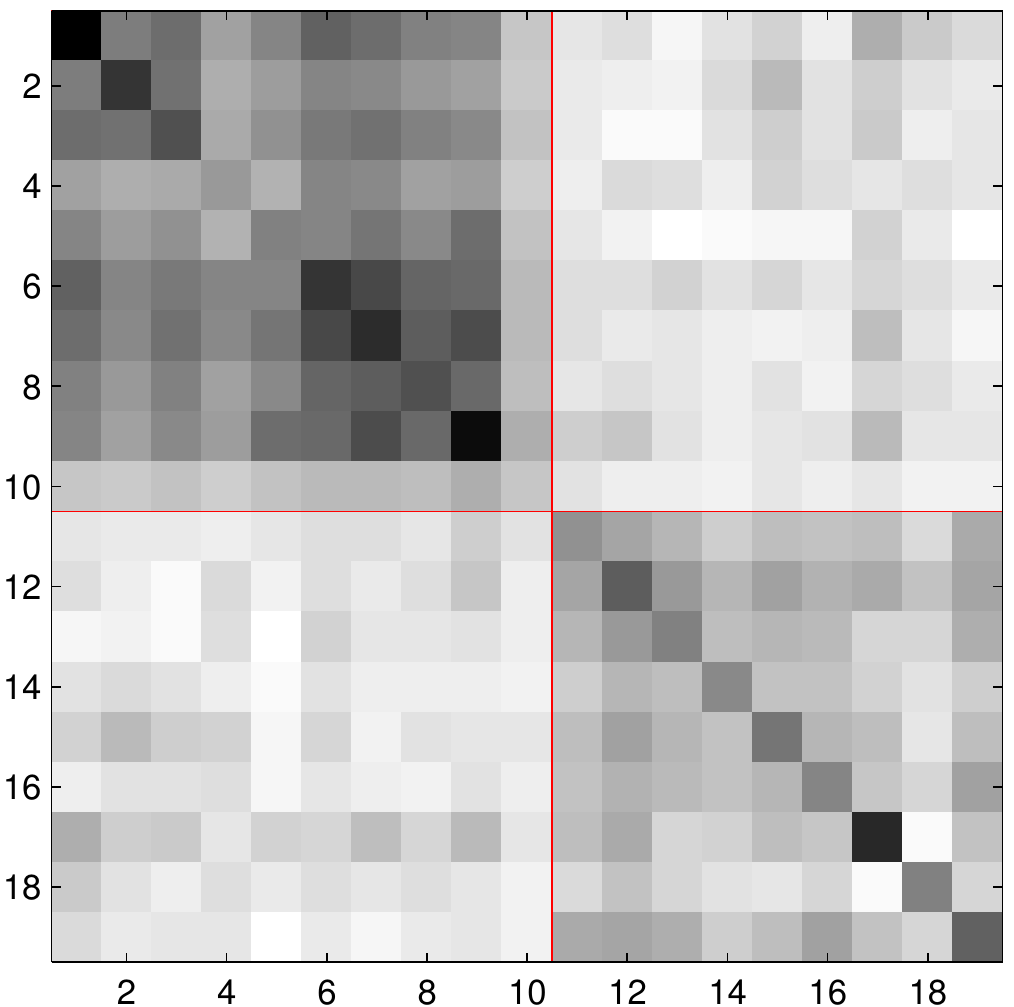}\hspace*{\fill}
\includegraphics[width=0.3\columnwidth]{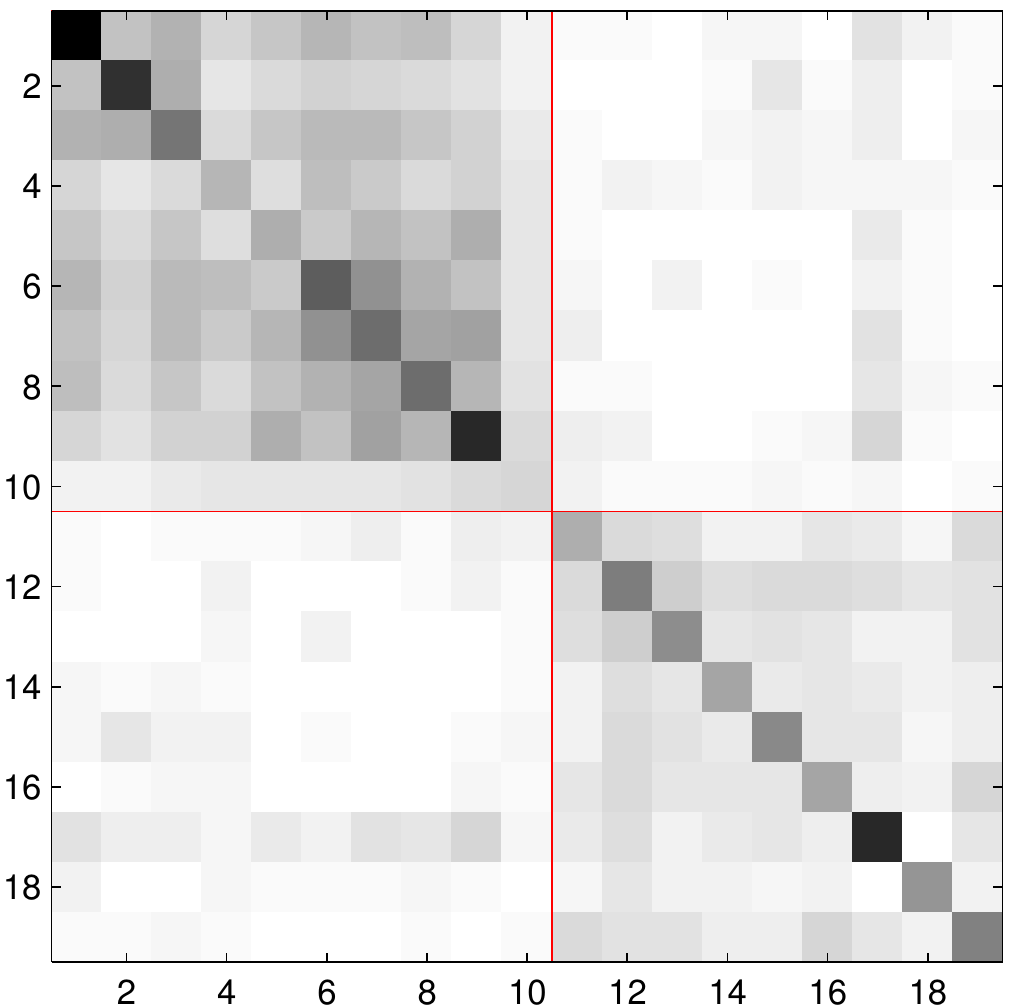}\hspace*{\fill}
\includegraphics[width=0.3\columnwidth]{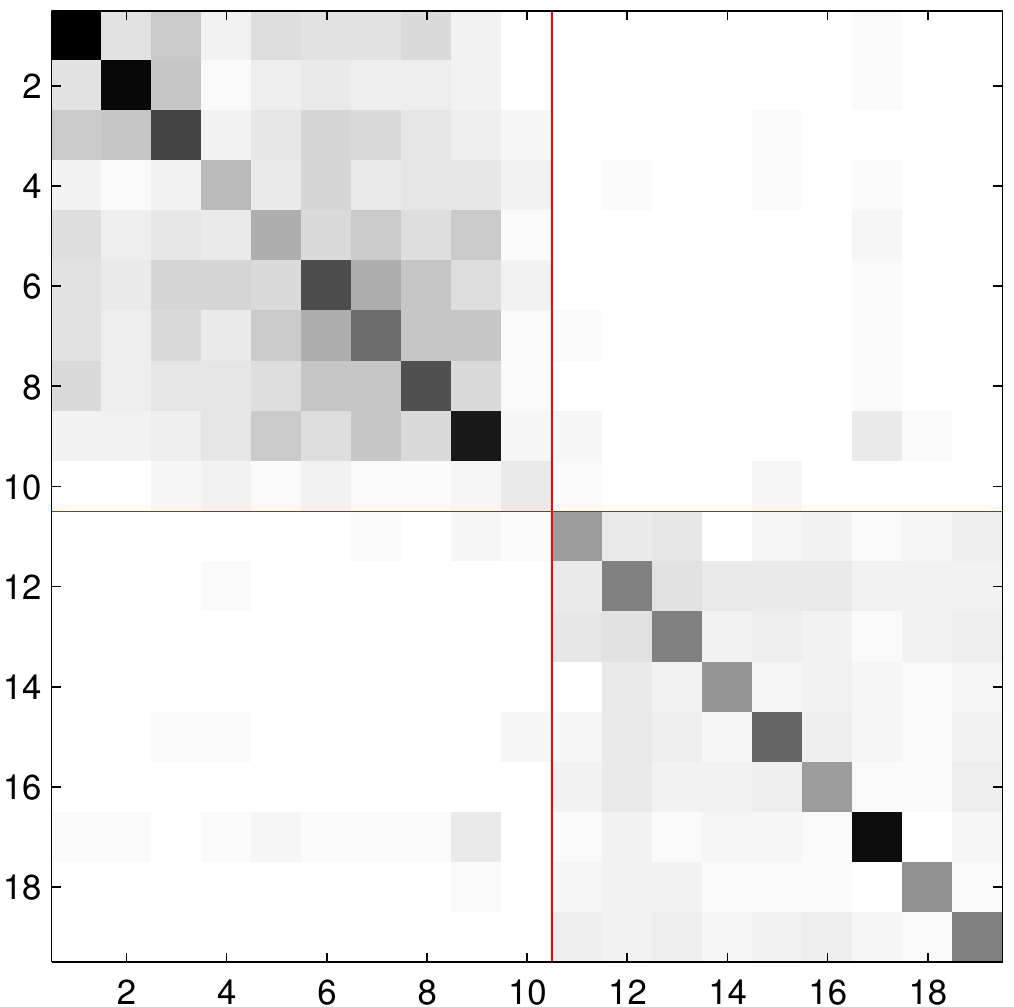}
{\hspace*{\fill}\ \ \ ($p=2$)\hspace*{\fill}\hspace*{\fill}\ \ \ \ \ \ \ \ \ \ \ \ \ \ ($p=4/3$)\hspace*{\fill}\hspace*{\fill}\ \ \ \ \ \ \ \ \ \ \ \ \  ($p=8/7$)\hspace*{\fill}  }
\caption{%
Plots of $|\Theta|$ matrices (rescaled to [0,1] and averaged over ten splits) computed by our solver \ourmethod for the Landmine data set for different $p$-norms, with cross-validated hyper-parameter values.  The darker regions indicate higher value. 
Tasks (landmines) numbered 1-10 correspond to highly foliated regions and those numbered 11-19 correspond to bare earth or desert regions. Hence, we expect two groups of tasks (indicated by the red squares). We can observe that the learned $\Theta$ matrix at $p=2$ depicts much more spurious task relationships than the ones at $p=4/3$ and $p=8/7$. Thus, our sparsifying regularizer improves interpretability. %Note that the plots (especially $p=8/7$) also suggest that the landmine region numbered 17 has more characteristics of being foliated than being desert.
% Note that the plots (especially $p=8/7$) also suggest that the task numbered 17 is related to both groups of tasks, i.e., the corresponding landmine has characteristics of both foliated and bare earth regions. %Best viewed in color.
\iflongversion
\else
% \vspace*{-20pt}
\fi
%Note that we did not use this hierarchical prior knowledge in our method.
}\label{fig:theta-landmine}%
\end{figure}
\subsection{Multi-Class Data Sets}\label{subsec:multi-class-exp}
% We also evaluated the our method in multi-class setting --- on handwritten digit recognition as well as scene classification data sets. 
The multi-class setup is cast as $T$ one-vs-all binary classification tasks, corresponding to $T$ classes. 
In this section we experimented with two loss functions: a) {\ourmethodHinge} -- the hinge loss employed in SVMs, and b) {\ourmethodSquare} -- the squared loss employed in OKL~\citep{Dinuzzo11}. In these experiments, we also compare our results with \textbf{MTL-SDCA}, a state-of-the-art multi-task feature learning method~\citep{Lapin14}. 
\iflongversion
In addition, we report results from our KL-divergence regularized formulation with squared loss (denoted by {\ourmethodKL}). 
\fi
% For our approach, we report results for two loss functions and with $q\in\{2,4,8\}$: a) \textbf{\ourmethodHinge} --- the hinge loss employed in SVMs, and b) \textbf{\ourmethodSquare} --- the squared loss, which was employed in OKL \citep{Dinuzzo11}. 
% The classification problem is cast into a multi-task framework by considering $T$ one-against-all binary classification tasks, corresponding to the $T$ classes.
% We begin the evaluation on handwritten digit recognition datasets and demonstrate the potential of the proposed approach. 
% We then proceed to perform an extensive evaluation on scene classification setting 
% where we achieve state-of-the-art results with features obtained 
% via a convolutional neural network (CNN)
% trained by~\citep{zhou2014learning}.
% After this, we present scalability experiment results in which we outperform the output kernel learning solver of~\citep{Dinuzzo11} (henceforth termed as OKL){\color{red}, whose formulation is a special case of the proposed FMTL$_q$ (with $q=2$) and squared loss function}.
% Multi-class classification

% \input{figures/table-mnist-usps_b.tex}
\iflongversion
\begin{table}[t]\centering\scriptsize
\caption{Mean accuracy and the standard deviation over five train-test splits. % on handwritten recognition data sets. 
%The proposed \ourmethodHingeq{8} achieves the best accuracy in both data sets. 
}\label{table:mnistusps}
% \vspace*{0.75em}
\renewcommand{\tabcolsep}{0.5em}
\begin{tabular}{l|c|ccc|cc|cc|c}\toprule
\multirow{2}{*}{\textbf{Data set}} &
\multirow{2}{*}{\textbf{STL}} &
\multirow{2}{*}{\textbf{MTL-SDCA}} &
\multirow{2}{*}{\textbf{GMTL}} &
\multirow{2}{*}{\textbf{MTRL}} &
\multicolumn{2}{c|}{\textbf{\ourmethodHinge}}&
\multicolumn{2}{c|}{\textbf{\ourmethodSquare}}&
\multirow{2}{*}{\textbf{\ourmethodKL}}\\
 & & & & & $p=2$ & $p=8/7$ & $p=2$ & $p=8/7$ & \\
\midrule
\midrule
MNIST & $84.1{\pm}0.3$ & $86.0{\pm}0.2$ & $84.8{\pm}0.3$ & $85.6{\pm}0.4$ & $86.1{\pm}0.4$ & $\mathbf{86.2{\pm}0.4}$ & $82.3{\pm}0.6$ & $82.4{\pm}0.3$ & $82.5{\pm}0.5$\\
USPS  & $90.5{\pm}0.3$ & $90.6{\pm}0.2$ & $91.6{\pm}0.3$ & $92.4{\pm}0.2$ & $92.4{\pm}0.2$ & $\mathbf{92.6{\pm}0.1}$ & $87.2{\pm}0.4$ & $87.5{\pm}0.3$ & $87.0{\pm}0.4$\\
% SUN397 ($m=5$) & $43.7\pm1.2$& $41.2\pm1.3$ & - & - & $41.5\pm1.1$ & $41.6\pm1.2$ & $\mathbf{44.1\pm1.3}$ & $44.0\pm1.2$ & $\mathbf{44.1\pm1.3}$\\
% SUN397 ($m=50$) & $\mathbf{58.6\pm0.1}$& $54.8\pm0.3$ & - & - & $55.1\pm0.2$ & $55.1\pm0.3$ & $\mathbf{58.6\pm0.1}$ & $\mathbf{58.6\pm0.2}$ & $58.4\pm0.1$\\
% STL & $84.1\pm0.3$ & $90.5\pm0.3$\\
% GMTL & $84.8\pm0.3$ & $91.6\pm0.3$ \\
% MTRL & $85.6\pm0.4$ & $\mathbf{92.4\pm0.2}$ \\
% MTL-SDCA & $86.0\pm0.2$ & $90.6\pm0.2$ \\
% \oursquared & $82.2\pm0.6$ & $87.3\pm0.4$ \\
% \ourhinge & $\mathbf{86.1\pm0.3}$ & $\mathbf{92.4\pm0.3}$ \\
\bottomrule
\end{tabular}
\end{table}

% \begin{table}[ht]\centering\small
% \caption{Performance on handwritten digit recognition datasets.
% }\label{table:mnistusps}
% \vspace*{0.75em}
% \begin{tabular}{l|c|c}\toprule
%  & MNIST & USPS\\
% \midrule
% \midrule
% STL & $84.1\pm0.3$ & $90.5\pm0.3$\\
% GMTL & $84.8\pm0.3$ & $91.6\pm0.3$ \\
% MTRL & $85.6\pm0.4$ & $\mathbf{92.4\pm0.2}$ \\
% MTL-SDCA & $86.0\pm0.2$ & $90.6\pm0.2$ \\
% \oursquared & $82.2\pm0.6$ & $87.3\pm0.4$ \\
% \ourhinge & $\mathbf{86.1\pm0.3}$ & $\mathbf{92.4\pm0.3}$ \\
% \bottomrule
% \end{tabular}
% \end{table}

\textbf{Handwritten Digit Recognition}: 
% \subsubsection{Handwritten Digit Recognition}\label{sec:mnist-usps}
We consider the following two data sets and  follow the experimental protocol detailed in~\citep{Kang11}.\\
\textbf{USPS}: A handwritten digit data sets with 10 classes~\citep{Hull94}.
We process the images using PCA and reduce the dimensionality to 87.
This retains almost $87\%$ of variance.\\
\textbf{MNIST}: Another handwritten digit data set with 10 classes~\citep{Lecun98}.
PCA is employed to reduce the dimensionality to 64.

We use 1000, 500 and 500 examples for training, validation and test respectively. Table~\ref{table:mnistusps} reports the average accuracy achieved by various methods on both  data sets over 5 splits. 
Our approach  \ourmethodHinge obtains better accuracy than GMTL, MTRL and MTL-SDCA~\citep{Lapin14} on both data sets. 
% In particular, \ourmethodHingeq{8}  obtains the best accuracy in both USPS and MNIST data sets.
% \input{figures/table-mnist-usps_b.tex}
%In particular, we outperform OKL which essentially signifies the importance of our generic framework. 
% We compare to the STL, GMTL, MTRL baselines introduced above, and to the \textbf{MTL-SDCA} multi-task feature learning method of \citep{Lapin14}.
\else
% \input{figures/table-mnist-usps_b_withoutKL.tex}
% \subsubsection{Handwritten Digit Recognition}\label{sec:mnist-usps}
\textbf{\textbf{USPS} \& \textbf{MNIST} Experiments}: 
We followed the experimental protocol detailed in~\citep{Kang11}. Results are tabulated in Table~\ref{table:mnistusps}. 
% For our approach, we report results for two loss functions: a) \textbf{\ourmethodHinge} -- the hinge loss, and b) \textbf{\ourmethodSquare} -- the squared loss, which was employed in OKL \citep{Dinuzzo11}. 
% \ourmethodHinge obtains the best result on both data sets. 
Our approach  \ourmethodHinge  obtains better accuracy than GMTL, MTRL and MTL-SDCA~\citep{Lapin14} on both data sets. 
% In particular, \ourmethodHingeq{8}  obtains the best accuracy in both USPS and MNIST data sets.
% \input{figures/table-mnist-usps_b.tex}
%In particular, we outperform OKL which essentially signifies the importance of our generic framework. 
% We compare to the STL, GMTL, MTRL baselines introduced above, and to the \textbf{MTL-SDCA} multi-task feature learning method of \citep{Lapin14}.
\fi
% \vspace{-5pt}

\textbf{MIT Indoor67 Experiments}: %\label{sec:indoor67}
\iflongversion
We also report results on the MIT Indoor67 benchmark
\citep{zhou2014learning} which covers 67 indoor scene categories with over 100 images per class.
\else
We report results on the MIT Indoor67 benchmark
\citep{zhou2014learning} which covers 67 indoor scene categories.
\fi
We use the train/test split ($80/20$ images per class) provided by the authors. % and similarly tune hyper-parameters by cross-validation. 
\iflongversion
\ourmethodSquare achieved the accuracy of ${73.1\%}$, ${73.1\%}$ and ${73.3\%}$ with $p=2,4/3$ and $8/7$ respectively. Our KL-divergence regularized approach \ourmethodKL obtained ${73.1\%}$. 
Note that these are better than the ones reported in~\citep{koskela2014convolutional} ($70.1\%$) and~\citep{zhou2014learning} ($68.24\%$). 
\else
\ourmethodSquare achieved the accuracy of ${73.3\%}$ with $p=8/7$. % while \ourmethodKL obtained $\mathbf{73.1\%}$. 
Note that 
this is better than the ones reported in~\citep{koskela2014convolutional} ($70.1\%$) and~\citep{zhou2014learning} ($68.24\%$). 
\fi

\iflongversion
\begin{table*}\centering
{\scriptsize
\caption{Mean accuracy and the standard deviation over ten train-test splits on SUN397.
}\label{table:sun397}
% \vspace*{0.75em}
\renewcommand{\tabcolsep}{0.3em}
\begin{tabular}{l|c|cc|ccc|ccc|c}\toprule
\multirow{2}{*}{\textbf{$m$}} &
\multirow{2}{*}{\textbf{STL}} &
\multirow{2}{*}{\textbf{MTL}} &
\multirow{2}{*}{\textbf{MTL-SDCA}} &
\multicolumn{3}{c|}{\textbf{\ourmethodHinge}}&
\multicolumn{3}{c}{\textbf{\ourmethodSquare}}&
\multirow{2}{*}{\textbf{\ourmethodKL}}\\
 & & & & $p=2$ & $p=4/3$ & $p=8/7$ & $p=2$ & $p=4/3$ & $p=8/7$ & \\
\midrule
\midrule
\ \ $5$ & $40.5{\pm}0.9$ & $42.0{\pm}1.4$ & $41.2{\pm}1.3$& $41.5{\pm}1.1$& $41.6{\pm}1.3$ & $41.6{\pm}1.2$ & $\mathbf{44.1{\pm}1.3}$& $\mathbf{44.1{\pm}1.1}$& $44.0{\pm}1.2$ & $\mathbf{44.1{\pm}1.3}$\\
% \ \ $5$ & $43.7{\pm}1.2$ & $42.0{\pm}1.4$ & $41.2{\pm}1.3$& $41.5{\pm}1.1$& $41.6{\pm}1.3$ & $41.6{\pm}1.2$ & $\mathbf{44.1{\pm}1.3}$& $\mathbf{44.1{\pm}1.1}$& $44.0{\pm}1.2$ & $\mathbf{44.1{\pm}1.3}$\\
%    $10$ & $46.6\pm0.5$& $46.7\pm0.5$& $46.9\pm0.4$& $47.2\pm0.4$ & $47.1\pm0.5$ & $49.5\pm0.5$         & $\mathbf{49.7\pm0.4}$& $49.6\pm0.5$ & $49.6\pm0.5$\\
%    $20$ & $51.2\pm0.5$& $51.1\pm0.5$& $51.0\pm0.6$& $51.8\pm0.6$ & $51.4\pm0.6$ & $54.2\pm0.4$         & $\mathbf{54.3\pm0.3}$& $54.2\pm0.3$ & $54.2\pm0.4$\\
%    $50$ & $\mathbf{58.6{\pm}0.1}$ & $57.0{\pm}0.2$ & $54.8{\pm}0.3$& $55.1{\pm}0.2$& $55.6{\pm}0.3$ & $55.1{\pm}0.3$ & $\mathbf{58.6{\pm}0.1}$& $58.5{\pm}0.1$         & $\mathbf{58.6{\pm}0.2}$ & $58.4{\pm}0.1$\\
   $50$ & $55.0{\pm}0.4$ & $57.0{\pm}0.2$ & $54.8{\pm}0.3$& $55.1{\pm}0.2$& $55.6{\pm}0.3$ & $55.1{\pm}0.3$ & $\mathbf{58.6{\pm}0.1}$& $58.5{\pm}0.1$         & $\mathbf{58.6{\pm}0.2}$ & $58.4{\pm}0.1$\\
\bottomrule
\end{tabular}
}
\end{table*}

\fi
\iflongversion 
\begin{figure}[t]\small\centering
\includegraphics[width=0.3\columnwidth]{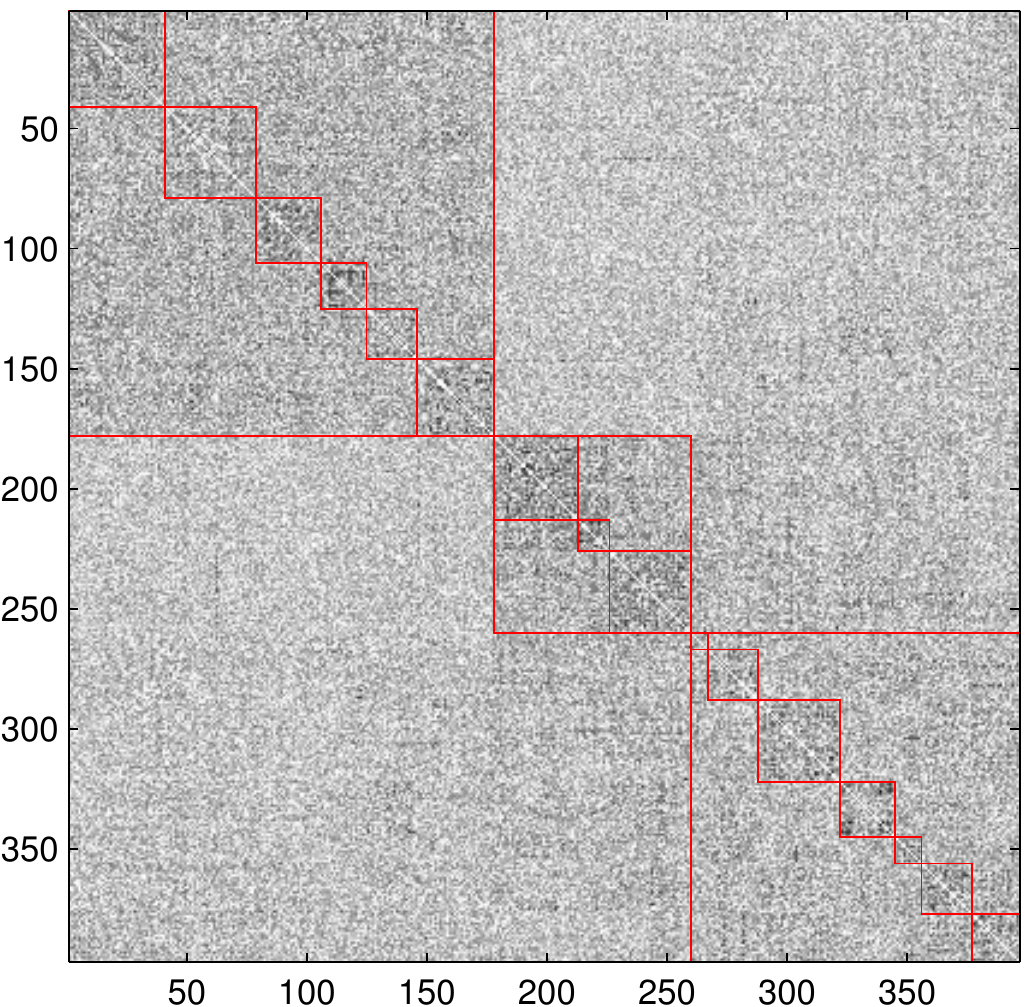}\hspace*{\fill}
\includegraphics[width=0.3\columnwidth]{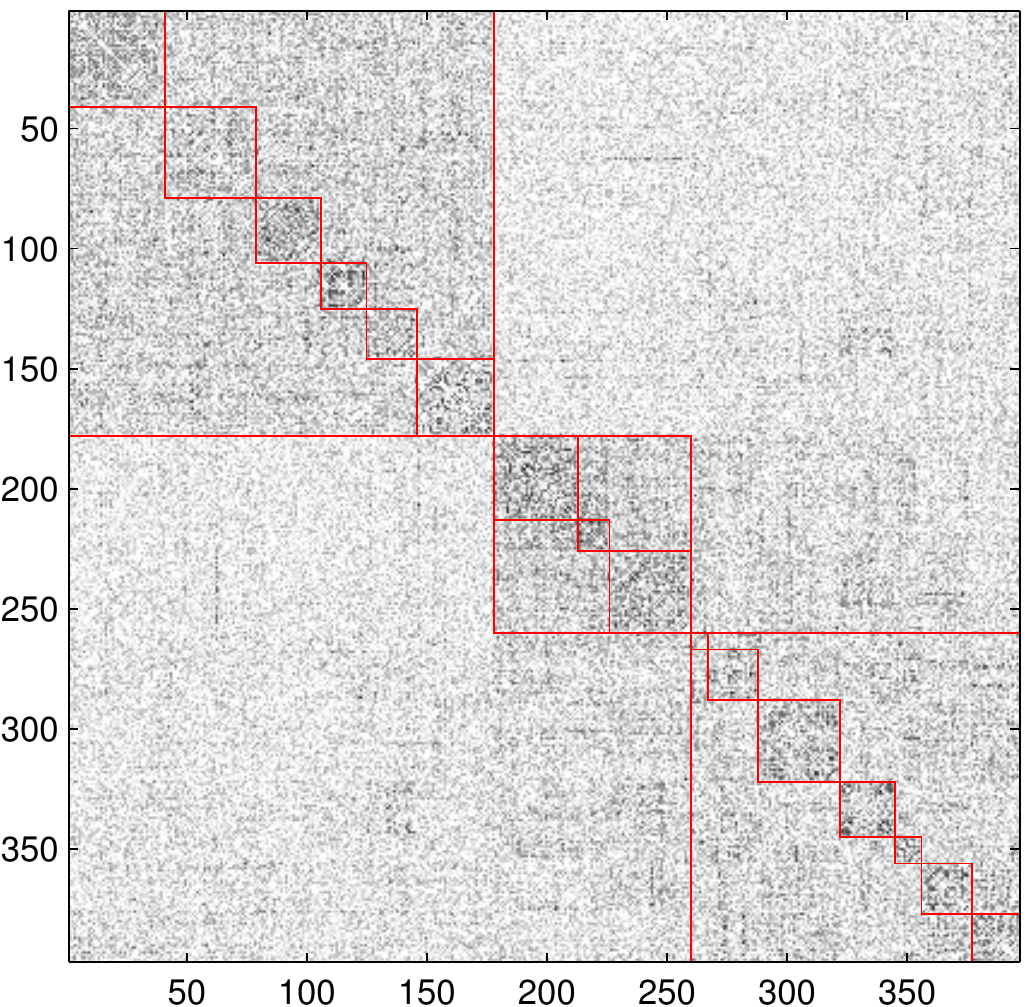}\hspace*{\fill}
\includegraphics[width=0.3\columnwidth]{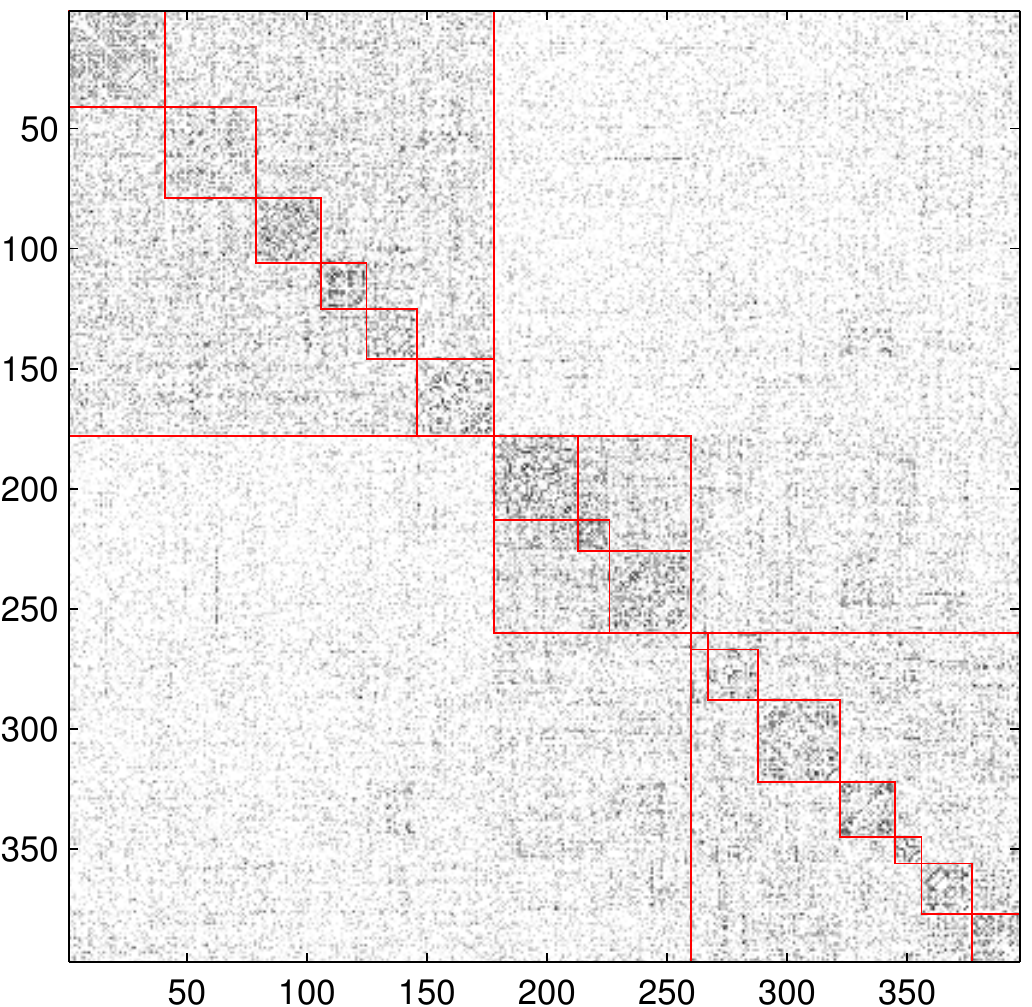}
{\hspace*{\fill}\ \ \ ($p=2$)\hspace*{\fill}\hspace*{\fill}\ \ \ \ \ \ \ \ \ \ \ \ \ \ ($p=4/3$)\hspace*{\fill}\hspace*{\fill}\ \ \ \ \ \ \ \ \ \ \ \ \  ($p=8/7$)\hspace*{\fill}  }
\caption{%
Plots of matrices $\log(1+|\Theta|)$ (rescaled to [0,1] and diagonal entries removed since they reflect high similarity of a task with itself, which is obvious) computed by our solver \ourmethodSquare
for the SUN397 data set for different $p$-norms, with cross-validated hyper-parameter values. 
The hierarchical block structure indicated by the red squares corresponds to the groups of classes available in SUN397, e.g., the top 3 super-classes are \textsl{indoor}, \textsl{outdoor-natural}, and \textsl{outdoor-man-made}, which in turn contain subgroups of classes. 
%Plots become sparser as $p$ decreases towards one. 
Note that this information was not used in experiments. We can observe that the learned $\Theta$ matrix at $p=2$ depicts much more spurious task relationships than the one at $p=8/7$. Thus, our sparsifying regularizer improves interpretability. Best viewed in color.
%Note that we did not use this hierarchical prior knowledge in our method.
}\label{fig:theta-2}%
\end{figure}
\fi 
% \input{figures/table-sun397.tex}
% \subsection{SUN397 Experiments}\label{sec:sun397}
\textbf{SUN397 Experiments}: 
\iflongversion 
SUN397~\citep{xiao2010sun} is a challenging scene classification
benchmark \citep{zhou2014learning} with 397 scene classes
 and more than 100 images per class.
\else
SUN397~\citep{xiao2010sun} is a challenging scene classification
benchmark \citep{zhou2014learning} with 397 classes. 
\fi 
% Until very recently \citep{zhou2014learning},
% this was the largest and most challenging scene classification
% benchmark ranging over 397 scene categories with more than 100 images per class.
% We follow the evaluation protocol of \citep{xiao2010sun}
% and use $m=5,10,20,50$ images per class for training and $50$ images per class for testing.
We use $m=5,50$ images per class for training, $50$ images per class for testing and report the average accuracy over the $10$ standard splits. 
\iflongversion
We employed the CNN features extracted with the convolutional neural network (CNN)
provided by \citep{zhou2014learning} using Places 205 database. 
We resized the images directly to $227\times227$
pixels, which is the size of the receptive field of that network.
The parameters were set by $2$-fold cross-validation. % in the range $10^{-3}, \ldots, 10^3$.
\else
We employed the CNN features extracted with the convolutional neural network (CNN)~\citep{zhou2014learning} using Places 205 database. 
\fi
The results are tabulated in Table~\ref{table:sun397}. 
% The available implementation of low rank output kernel approach, MTRL, is not efficient for the multi-class setting which leads to memory problems for SUN397. 
\iflongversion

Figure~\ref{fig:theta-2} offers a qualitative assessment of the proposed method by showing the output kernel matrices $\Theta$ computed by our formulation  \ourmethodSquare for various $p$-norms. We can observe that the $\Theta$ matrix becomes sparser as the $p$-norm decreases from $2$ towards one. 
Enforcing sparsity helps to detect the hierarchical structure of the tasks (see caption for more details). 
% The underlying hierarchical structure among the tasks is more clearly understood from the $\Theta$ matrix with $p=8/7$ than with $p=2$. 
% Please refer to the caption for more details.
\else
The $\Theta$ matrices computed by \ourmethodSquare are discussed in the supplementary material.% (along with additional experimental results).
\fi
% \ourmethodSquare obtained $44.1\%$ accuracy with $p=2,4/3$ and $44.0\%$ with $p=8/7$. Baseline results were: STL $(43.7\%)$, MTL-SDCA $(41.2\%)$ and MTL $(42.0\%)$. 
% Our method  \ourmethodSquare and \ourmethodKL outperform both STL and {MTL-SDCA}, a state-of-the-art multi-task approach on scene classification~\citep{Lapin14}. 
% Note that the latter result is better than those reported in~\citep{zhou2014learning} ($54.3\pm0.1$ with Places 205) and~\citep{koskela2014convolutional} ($54.7\pm0.2$ with ImageNet). 
% Its observation again shows that sparsity helps in removing spurious correlations among the tasks. 
% In the latter case, the $q=4,8$ norms shows significant improvement over $q=2$ norm when the training data size is small. 
% In such data scarce regimes, there is a greater chance to learn spurious correlations among the tasks. 
% The results confirm that the sparsity inducing effect of the generic regularizer, when compared with $q=2$, was able to identify and reject more spurious correlations among the tasks. 
% We note that we outperform \citep{zhou2014learning} even with the STL baseline, which can be explained by proper model selection in our experiments (they used the default value $C=1$ for SVM training).
% 

% The experimental results are reported in Table~\ref{tbl:indoor67}.
% As before, we are able to outperform \citep{zhou2014learning} achieving $71.5\%$ with the STL baseline.
% This result is further improved by the proposed \ourmethodSquare method which again obtains the state of the art result of $\mathbf{73.4\%}$.

\iflongversion
\begin{figure}[H]\centering
{\includegraphics[width=0.37\columnwidth]{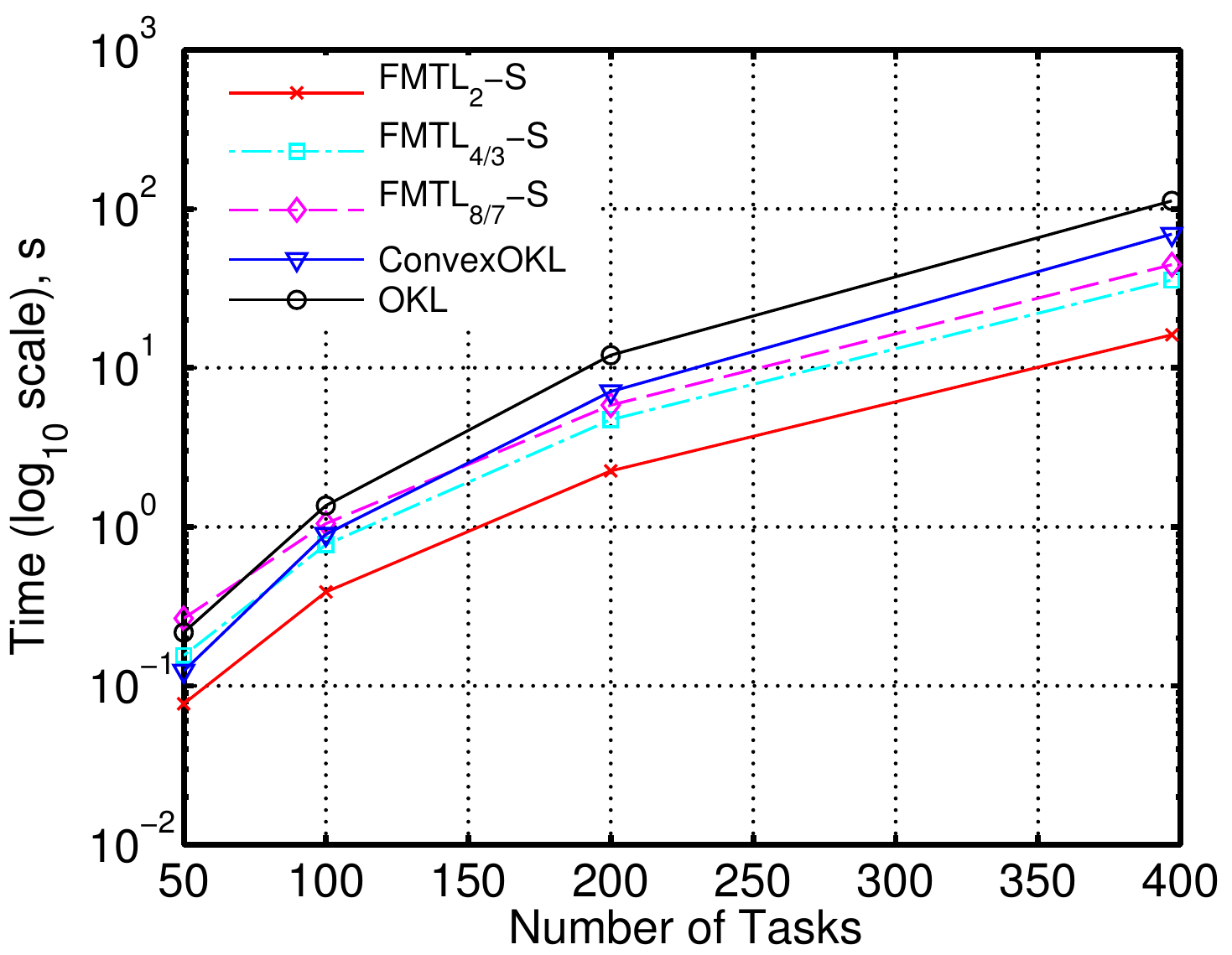}}\ \ \ \ \ 
\else
\begin{figure}\centering
{\includegraphics[width=0.37\columnwidth]{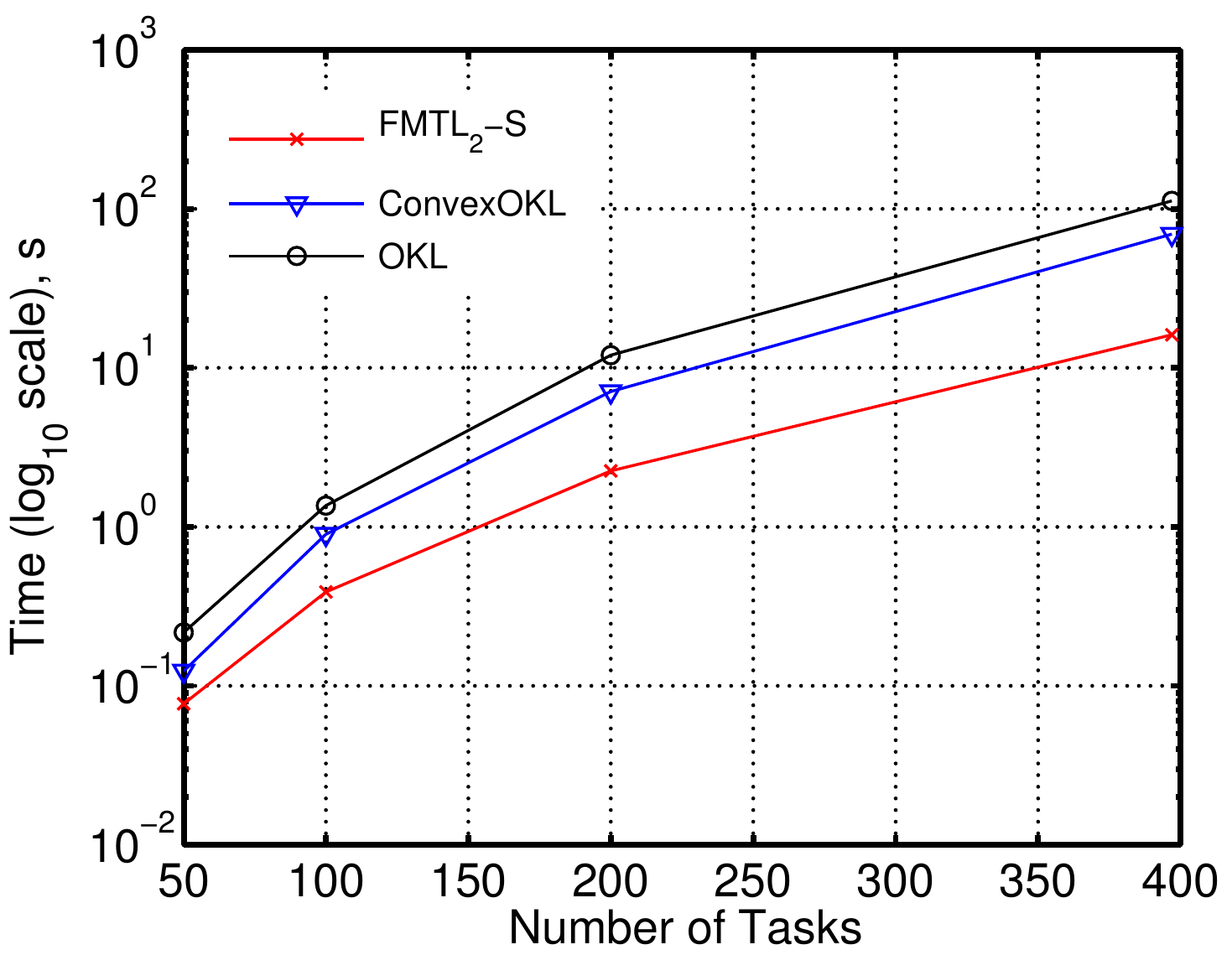}}\ \ \ \ \ 
\fi
\includegraphics[width=0.36\columnwidth]{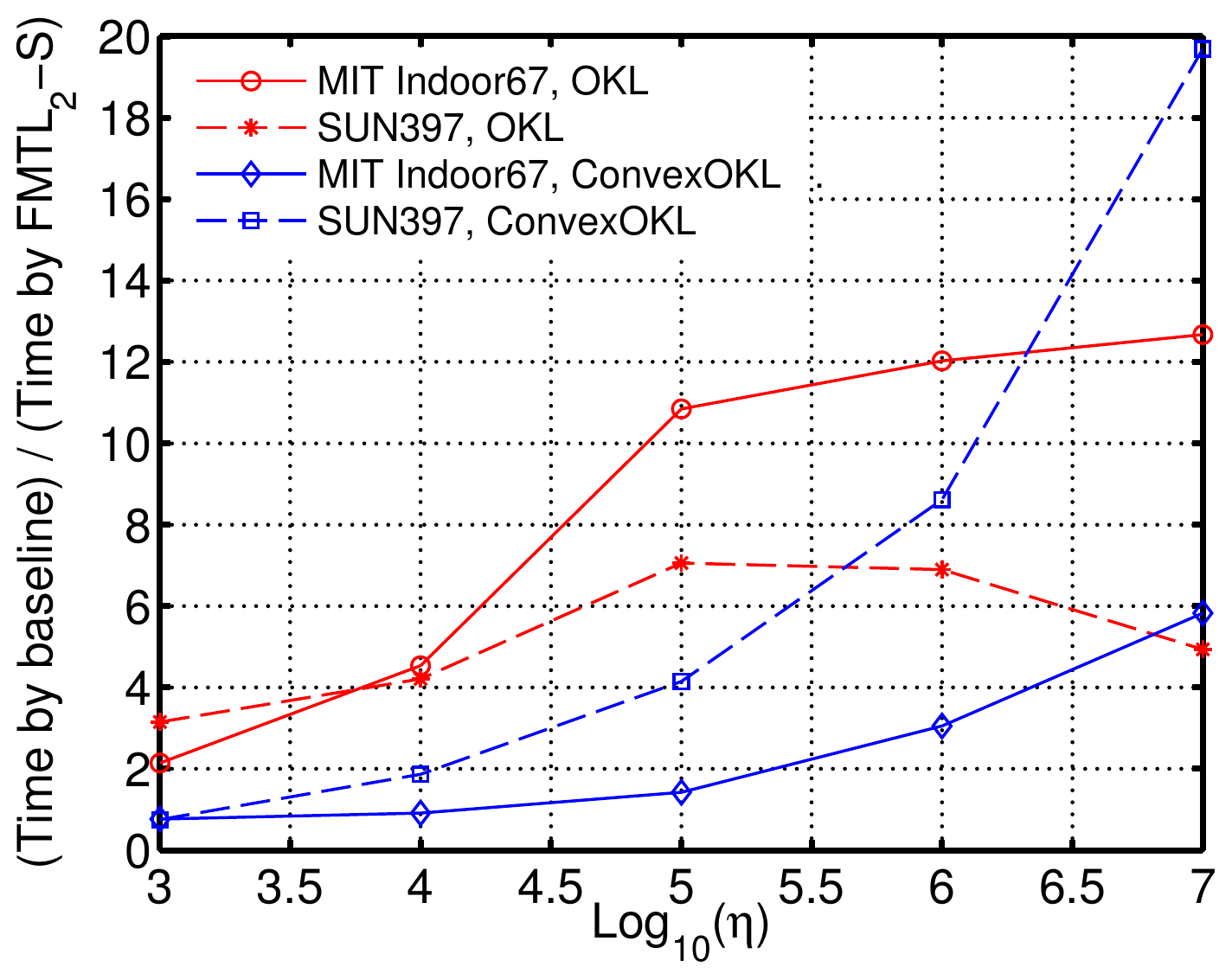}\\%
{\hspace*{\fill}\ \ \ \ \ \ \ \ \ \ \ \ \ \ \ \ \ \ \ \ (a)\hspace*{\fill}\hspace*{\fill}(b)\hspace*{\fill} \ \ \ \ \ \ \ \ \ \ \ }
\caption{
(a) Plot compares the runtime of various algorithms with varying number of tasks on SUN397. Our approach \oursquared is 7 times faster that OKL~\citep{Dinuzzo11} and 4.3 times faster than ConvexOKL~\citep{Ciliberto15} when the number of tasks is maximum. 
\iflongversion
It can be observed that \oursquared also has the best computational complexity in terms of number of tasks. 
\fi
(b) Plot showing the factor by which \oursquared outperforms OKL and ConvexOKL over the hyper-parameter range on various data sets. % with maximum number of tasks. 
On SUN397, we outperform OKL and ConvexOKL by factors of $5.2$ and $7$ respectively. On MIT Indoor67, we are better than OKL and ConvexOKL by factors of $8.4$ and $2.4$ respectively.
% \iflongversion
% As mentioned in Section~\ref{sec:specialDual}, the number of hyper-parameters in~\eqref{eq:final-dual} is reduced to one. The expression of $\eta$ is given by $\eta=\frac{\lambda C^3}{(4k-2)} \Big(\frac{2k-1}{2k\lambda}\Big)^{2k}$ for general $p$-norm, where $p=\frac{2k}{2k-1}$.
% \fi
}\label{fig:scaling-cnn}
\end{figure}
\iflongversion
\begin{figure}[H]\centering
\includegraphics[width=0.36\columnwidth]{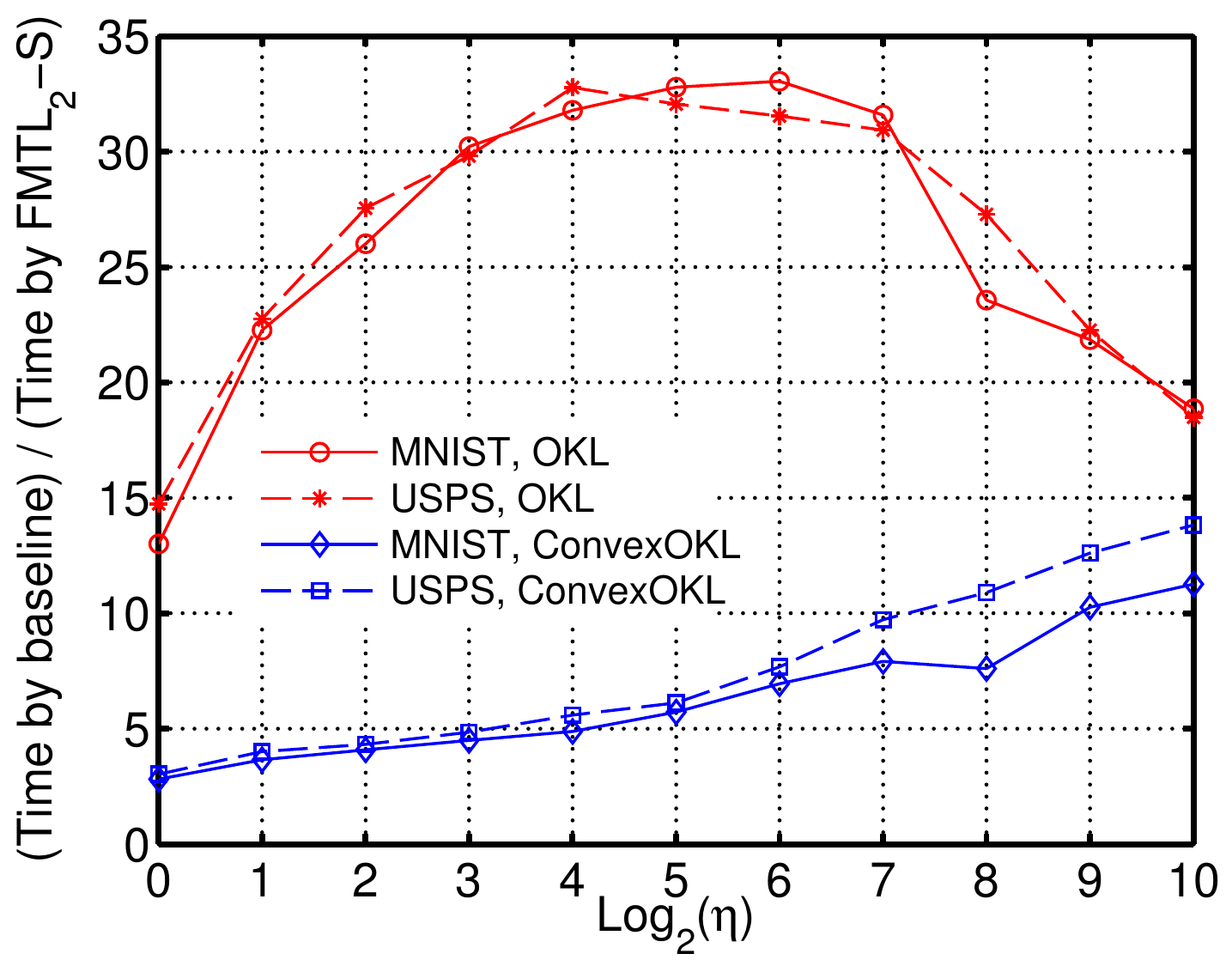}\ \ \ \ \ 
\includegraphics[width=0.37\columnwidth]{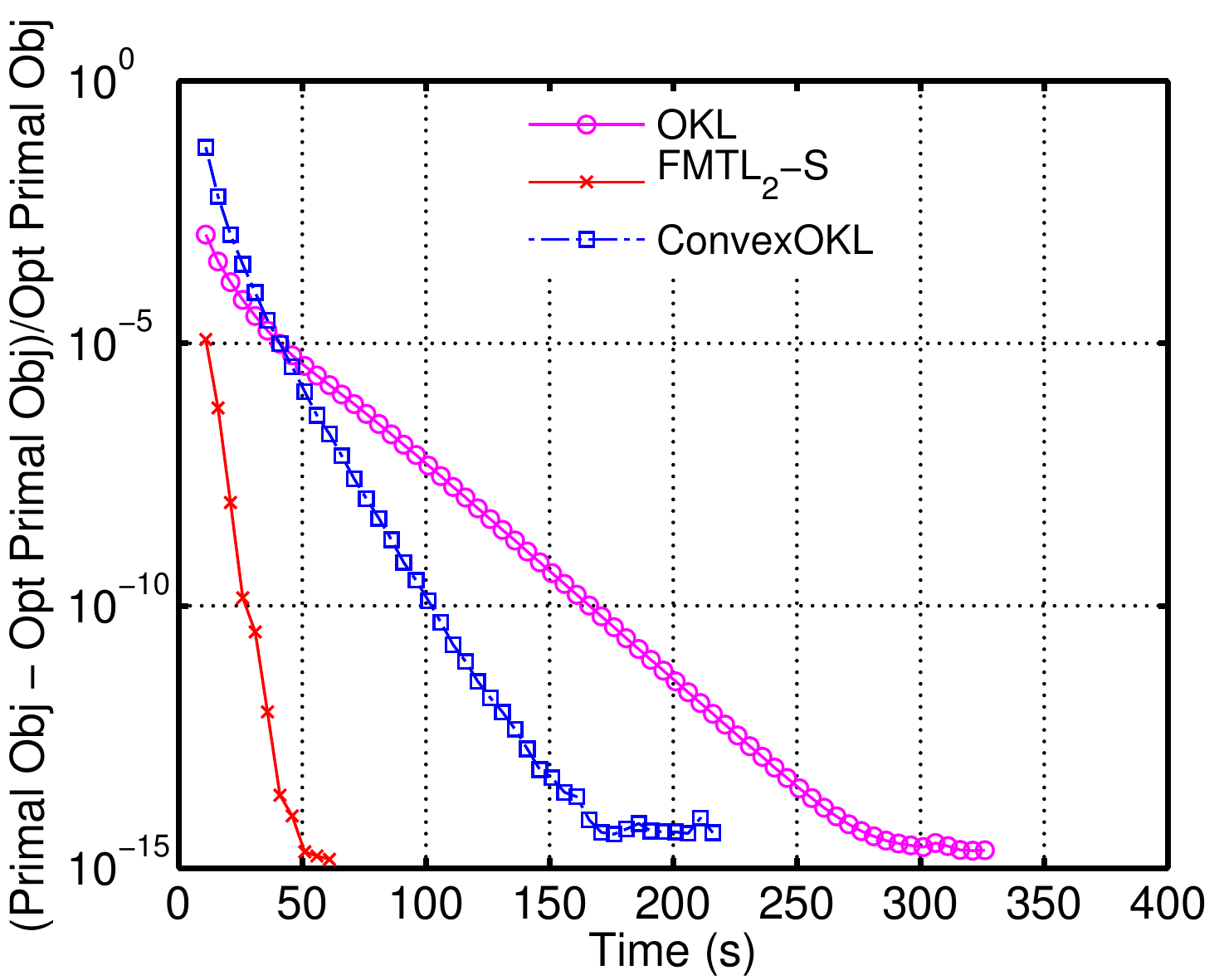}\\%
% {\hspace*{\fill}\ \ \ \ \ \ \ \ \ \ \ (a)\hspace*{\fill}\hspace*{\fill}(b)\hspace*{\fill} \ \ \ }
{\hspace*{\fill}\ \ \ \ \ \ \ \ \ \ \ \ \ \ \ \ \ \ \ \ (a)\hspace*{\fill}\hspace*{\fill}(b)\hspace*{\fill} \ \ \ \ \ \ \ \ \ \ \ }
\caption{(a)Plot showing the factor by which \oursquared outperforms OKL and ConvexOKL over the hyper-parameter range on MNIST and USPS data sets. % with maximum number of tasks. 
On MNIST, we outperform OKL and ConvexOKL by factors of $25.5$ and $6.3$ respectively. On USPS, we are better than OKL and ConvexOKL by factors of $26.2$ and $7.4$ respectively. (b) Plot comparing the average rate at which the three algorithms achieve the optimal primal objective (on SUN397 data set). \oursquared was run with a duality gap of $10^{-15}$ and its primal objective value was taken to be the optimal primal objective.
}\label{fig:eta-mnist-usps}
\end{figure}
\fi
\subsection{Scaling Experiment}\label{sec:scaling}
We compare the runtime of our solver for \oursquared with the OKL solver of \citep{Dinuzzo11} and the ConvexOKL solver of \citep{Ciliberto15} on several data sets. 
All the three methods solve the same optimization problem. 
Figure~\ref{fig:scaling-cnn}a shows the result of the scaling experiment where we vary the number of tasks (classes). The parameters employed are the ones obtained via cross-validation. 
Note that both OKL and ConvexOKL algorithms do not have a well defined stopping criterion whereas our approach can easily compute the relative duality gap (set as $10^{-3}$). We terminate them when they reach the primal objective value achieved by \oursquared. 
Our optimization approach is $7$ times and $4.3$ times faster than the alternate minimization based OKL and ConvexOKL, respectively, when the number of tasks is maximal. 
The generic FMTL$_{p=4/3,8/7}$ are also considerably faster than OKL and ConvexOKL. % (discussed in the supplementary material).
\iflongversion

\fi

Figure~\ref{fig:scaling-cnn}b compares the average runtime of our \ourmethodSquare with OKL and ConvexOKL on the cross-validated range of hyper-parameter values. 
\iflongversion
 The hyper-parameter value chosen by cross-validation for SUN397 and MIT Indoor67 data sets was around $10^5$. 
\fi
\ourmethodSquare outperform them on both MIT Indoor67 and SUN397 data sets. %Please refer to the caption for the details. 
% We outperform OKL on average by factors of $8.4$ and $5.2$ on MIT Indoor67 and SUN397 data sets respectively, with maximum number of tasks. 
\iflongversion
Figure~\ref{fig:eta-mnist-usps} shows the same comparison on MNIST and USPS data sets. 
\else
On MNIST and USPS data sets, \ourmethodSquare is more than $25$ times faster than OKL, and more than $6$ times faster than ConvexOKL. Additional details of the  above experiments are discussed in the supplementary material. 
\fi

% \begin{figure}[ht]
% \vskip 0.2in
% \begin{center}
% \centering{
% \includegraphics[width=0.6\linewidth]{figures/Caltech101_time_comparison_minus_one_log.pdf}
% % \hspace*{\fill}
% % \includegraphics[width=0.6\linewidth]{figures/Caltech256_scale_classes.pdf}
% \includegraphics[width=0.6\linewidth]{figures/Caltech256_scale_classes_log.pdf}
% }
% \caption{Comparison of our optimization approach and the alternate minimization approach of~\citep{Dinuzzo11} on Caltech datasets. (Figure will be improved if included in the final draft)}
% \label{fig:caltech}
% \end{center}
% \end{figure}

%%%%%%%%%%%%%%%%%%%%%%%%%%%%%%%%%%%%%%%%%%%%%%%%%%%%%%%%%%%%%%%%%%%%%%%%%%%%%%%%%%%%%%%%%%%%%%%%%%%%%%%%%%%%%%%%%%%%%%%%%%%%%%%%%%%%%%%%%%%%%%%%%%%%%%%%%%%%%
%%%%%%%%%%%%%%%%%%%%%%%%%%%%%%%%%%%%%%%%%%%%%%%%%%%%%%%%%%%%%%%%%%%%%%%%%%%%%%%%%%%%%%%%%%%%%%%%%%%%%%%%%%%%%%%%%%%%%%%%%%%%%%%%%%%%%%%%%%%%%%%%%%%%%%%%%%%%%
%%%%%%%%%%%%%%%%%%%%%%%%%%%%%%%%%%%%%%%%%%%%%%%%%%%%%%%%%%%%%%%%%%%%%%%%%%%%%%%%%%%%%%%%%%%%%%%%%%%%%%%%%%%%%%%%%%%%%%%%%%%%%%%%%%%%%%%%%%%%%%%%%%%%%%%%%%%%%

\section{Conclusion}\label{sec:conclusion}
We proposed a novel formulation for learning the positive semi-definite output kernel matrix for multiple tasks. Our main technical contribution is our analysis of a certain class of regularizers on the output kernel matrix where one may drop the positive semi-definite constraint from the optimization problem, but still solve the problem optimally. 
This leads to a dual formulation that can be efficiently solved using stochastic dual coordinate ascent algorithm. Results on benchmark multi-task and multi-class data sets demonstrates the effectiveness of the proposed multi-task algorithm in terms of runtime as well as generalization accuracy. %Possible future work includes characterization of the complete class of such regularizers where our analysis holds.

\textbf{Acknowledgments.} P.J. and M.H. acknowledge the support by the Cluster of Excellence (MMCI).

{\small
\bibliographystyle{unsrt}
\bibliography{main_bib}
}
\end{document}